\documentclass[11pt]{article}

\usepackage[utf8]{inputenc}
\usepackage[T1]{fontenc}
\usepackage[letterpaper,margin=1in]{geometry}
\usepackage[square,numbers,sort&compress]{natbib}
\usepackage[hidelinks]{hyperref}
\usepackage{microtype}

\usepackage{amsmath, amssymb, amsthm, mathtools}
\usepackage{graphicx}
\usepackage{booktabs}
\usepackage{multirow}
\usepackage{algorithm}
\usepackage{algpseudocode}
\usepackage{float}      
\usepackage{caption}    
\usepackage{wrapfig}    
\usepackage{subcaption} 
\usepackage{cleveref}
\usepackage{enumitem}
\usepackage{xcolor}
\usepackage{tikz}
\usetikzlibrary{arrows.meta, positioning, calc, shapes.geometric}

\theoremstyle{plain}
\newtheorem{theorem}{Theorem}
\newtheorem{proposition}[theorem]{Proposition}
\newtheorem{lemma}[theorem]{Lemma}
\newtheorem{corollary}[theorem]{Corollary}
\theoremstyle{definition}
\newtheorem{definition}[theorem]{Definition}
\newtheorem{assumption}[theorem]{Assumption}
\theoremstyle{remark}

\newcommand{\R}{\mathbb{R}}
\newcommand{\N}{\mathbb{N}}

\newcommand{\Prob}{\mathbb{P}}
\newcommand{\sX}{\mathcal{X}}
\newcommand{\sL}{\mathcal{L}}

\newcommand{\sF}{\mathcal{F}}
\newcommand{\sC}{\mathcal{C}}
\newcommand{\sZ}{\mathcal{Z}}

\newcommand{\cpz}{\mathrm{CPZ}}
\newcommand{\opand}{\,\wedge\,}
\newcommand{\opor}{\,\vee\,}
\newcommand{\opnot}{\neg}
\newcommand{\primfam}[1]{\ensuremath{\mathsf{#1}}}
\newcommand{\dH}{d_{\mathrm{H}}}
\newcommand{\Hdist}[2]{\dH(#1, #2)}
\newcommand{\keywords}[1]{\par\medskip\noindent\textbf{Keywords:} #1\par}

\title{CompCPZ: Preserving Multi-Modal Intent in\\
       Language-Guided Robot Manipulation}

\author{
  Zhen Zhang \quad Ahmad Hafez \quad Peng Xie\\
  Yanliang Huang \quad Wenyuan Wu \quad Amr Alanwar\\[0.5em]
  {\small School of Computation, Information and Technology}\\
  {\small Technical University of Munich, Germany}\\
  {\footnotesize\texttt{\{zhenzhang.zhang, a.hafez, p.xie, yanliang.huang, wenyuan.wu, alanwar\}@tum.de}}
}
\date{}

\begin{document}
\maketitle

\begin{abstract}
A robot asked to ``place the cup near the red plate or the blue
plate'' may reach the centroid between them and appear
geometrically successful, while satisfying neither disjunct of
the instruction. This silent semantic failure exposes a
structural limitation of language-conditioned robot policies:
representations that collapse a disjunctive instruction into a
single connected set cannot preserve all feasible modes, and
planners that commit to one action degrade under run-time mode
uncertainty. We address this limitation with CompCPZ, a sound
algebraic layer that language-conditioned learning systems wrap
to recover multi-modal disjunctive representation, recursively
composing per-primitive constrained polynomial zonotope
enclosures along the language parse tree with distribution-free
conformal coverage and sub-millisecond runtime. On a closed-loop
ManiSkill3 tabletop-manipulation benchmark, CompCPZ outperforms
convex set baselines, multi-peak decoders, and a zero-shot
vision-language-action model ($1{,}900/1{,}918$ paired wins,
$p \ll 10^{-30}$); the same compiler also transfers without
retuning to planar real-robot trials on a Unitree Go2 quadruped
under motion capture.
These results suggest that compositional language grounding
should be evaluated not only by reaching a decoded target, but by
whether the represented feasibility set preserves the
connected-component structure of the user's intent.
\end{abstract}


\section{Introduction}
\label{sec:intro}

A Franka arm asked to ``put the cube near the red mug or the blue
plate'' is routinely demonstrated to reach the centroid between the
two targets and report geometric success, while its end-effector
lands in empty air that satisfies neither disjunct. We call
this a \emph{silent semantic} failure:
the action is geometrically coherent in isolation, yet violates the
disjunctive structure of the user's instruction. As
vision--language--action pipelines scale to broader instruction
sets, this becomes a load-bearing limitation for these pipelines: any
system whose geometric output collapses disjunctive intent fails
compositionally. Today's vision-language-to-action pipelines
either decode disjunctive instructions as a single connected
feasible region, as non-sound multi-peak
heatmaps or value maps~\citep{huang2023voxposer,zhou2025physvlm},
or commit to a single executed
action~\citep{kim2024openvla,brohan2023rt2,black2024pi0,chi2023diffusion}.
We prove both failures are structural, not tuning: any sound
encoding whose component count is smaller than the number of modes
misses the true feasibility set by at least half the inter-mode
separation (a structural floor on connected-component count, not
on representation capacity), and any single-action planner under
run-time mode uncertainty has worst-case success that decays with
the number of modes. Non-sound outputs
(value maps, heatmaps, single-point actions) fall outside the lower
bound's scope and are treated empirically.

\begin{figure}[!ht]
\centering
\includegraphics[width=\linewidth]{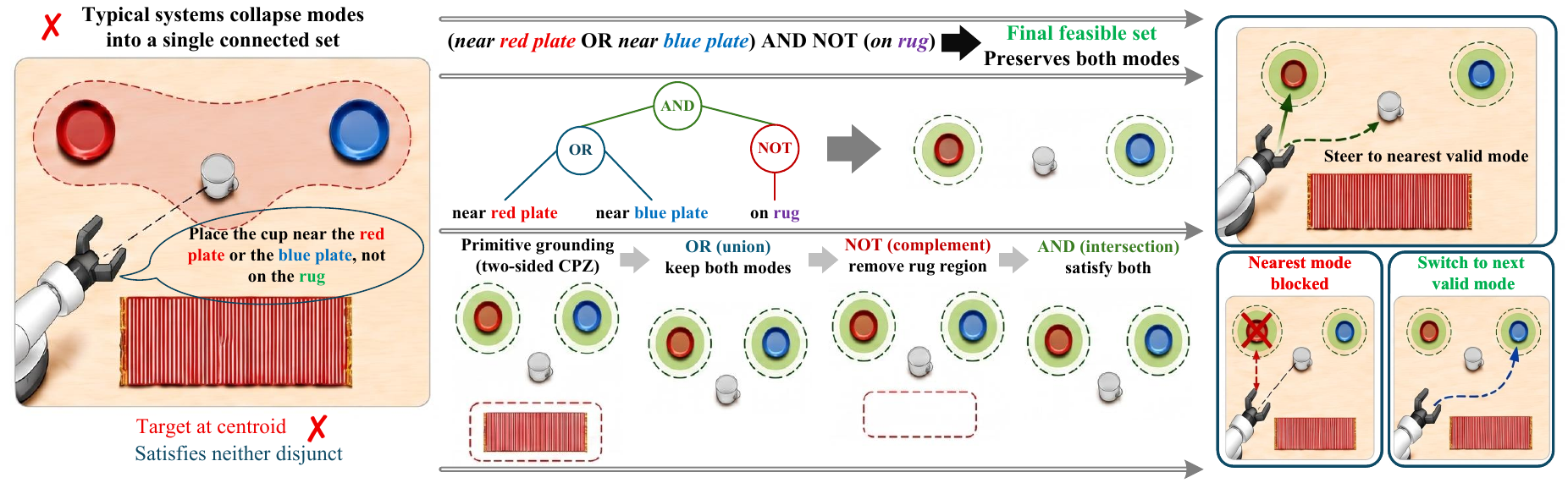}
\caption{CompCPZ compiles a natural-language instruction into a
sound $k$-component feasibility set via bottom-up Boolean
composition of two-sided per-primitive enclosures, with run-time
mode switching for handling occlusion of the nearest mode.}
\label{fig:pipeline}
\end{figure}

We meet the lower bound with CompCPZ (\Cref{fig:pipeline}), which
parses compositional language into a primitive grammar and
recursively composes constrained polynomial zonotopes to realise
Boolean intersection, union, and (on axis-aligned grammars) exact
complement, with distribution-free conformal coverage and
sub-millisecond runtime.
The pipeline is evaluated on simulated ManiSkill3 tabletop
manipulation and on planar real-robot trials on a Unitree Go2
quadruped.

We summarise three contributions. First, a structural impossibility
result: every sound grounding with fewer components than the
instruction has modes misses the true set by at least half the
inter-mode separation, complemented by a single-action lower bound.
Second, to our knowledge the first compositional compiler from
natural-language parse trees to sound $k$-component set
enclosures, with
distribution-free conformal coverage propagation; we instantiate
it with constrained polynomial zonotopes, a representation that
also enables exact axis-aligned complement, polynomial primitives,
depth-independent Lipschitz stability, and containment-checked
run-time mode switching, none of which are soundly supported by
hand-rolled union-of-AABBs encodings. Third, a closed-loop
ManiSkill3 benchmark and planar real-robot trials on a Unitree
Go2 quadruped: CompCPZ wins $1{,}900/1{,}918$ paired sign
tests at $p \ll 10^{-30}$ against convex baselines, an oracle
GMM, multi-peak decoders, and a zero-shot
vision--language--action model on the regimes where the
topological bound is non-trivial. Together
these results frame compositional language grounding for robotics
as a sound-representation problem that learning systems wrap to
preserve, not of clean-benchmark success rate.

CompCPZ is a sound algebraic foundation for language-conditioned
robot policies, used as a wrapping layer rather than a new
learning algorithm. Three learning components feed the
algebra layer: a vision--language parser (GPT-4o, chain-of-thought
prompting); an in-domain fine-tuned visual detector (YOLOv8n,
$200$ synthetic frames, $3$\,min on one GPU); and
distribution-free split-conformal calibration of the per-primitive
thickness. The detector lifts end-to-end in-GT-mode from $0.20$
with off-the-shelf
Grounding-DINO+SAM2~\citep{liu2023grounding,ravi2024sam2} to
$1.00$ (\Cref{tab:training_data_ablation}). The topological
lower bound (\Cref{thm:topologicalLowerBound}) applies to
\emph{any} sound encoder above this layer, handcrafted or learned
(value maps, single-decoder VLA policies, regression-style
end-effector heads); no amount of training fixes a
single-component representation. The contribution to robot
learning is therefore not a new training algorithm but the
principled algebraic layer these four policy classes
structurally need to overcome the topological lower bound and
propagate distribution-free uncertainty through Boolean
composition.

\section{Related Work}
\label{sec:related}

\paragraph{VLM-to-constraint pipelines.}
We organise prior work by the geometry of its output, since this
is what our topological bound acts on. (i)~Sound convex
enclosures (axis-aligned bounding boxes, minimum-volume enclosing
ellipsoids) over identified objects are covered by
\Cref{thm:convexLowerBound} and benchmarked here as AABB / MVEE
baselines. (ii)~Dense scalar fields / pointwise
affordances: value maps~\citep{huang2023voxposer}, reachability
heatmaps~\citep{zhou2025physvlm}, affordance
points~\citep{yuan2024robopoint}, and code piped into
TAMP~\citep{owltamp2024} produce non-sound multi-peak outputs that
fall outside the topological bound; we add heatmap-argmax and
sampling-decoder baselines to cover them empirically.
(iii)~Single committed actions:
VLAs~\citep{kim2024openvla,brohan2023rt2} output a
single 7-DoF action and flow/diffusion
policies~\citep{black2024pi0,chi2023diffusion} a fixed action chunk; either
way the committed output is a finite-dimensional point that
cannot contain a positive-volume disjunctive feasibility set,
and we benchmark OpenVLA-7B and formalise this single-point
failure as \Cref{prop:singlePoint}.

\paragraph{Set representations, safety, and conformal prediction.}
CompCPZ extends the constrained polynomial zonotope
algebra~\citep{kochdumper2023constrained,zhang2025data}, which
generalises zonotopes~\citep{girard2005reachability}, constrained
zonotopes~\citep{scott2016constrained}, and polynomial
zonotopes~\citep{kochdumper2020sparse}. Hybrid zonotopes~\citep{bird2023hybrid} exactly capture unions of
polytopes (our CPZUnion of \Cref{app:union} reduces to one in
the axis-aligned case) but cannot represent the polynomial
primitives (ellipsoid avoidance, between-region quadratic
constraints) admitted by our grammar.
Control barrier function safety
filters~\citep{ames2019cbf,wabersich2023data}, including recent
plug-and-play VLA safety layers~\citep{vlsa2025}, enforce
pointwise constraint satisfaction at execution time but do not
compose along a language parse tree at the representation level.
Conformal prediction~\citep{vovk2005algorithmic,angelopoulos2023gentle}
is wrapped per-primitive and union-bounded over leaves, with
explicit DKW sample complexity in \Cref{app:proofs}.
ManiSkill3~\citep{tao2025maniskill3} hosts our closed-loop benchmark
at $10\times$ more seeds per cell than comparable
LIBERO~\citep{liu2024libero} settings.

\section{Method}
\label{sec:method}

\subsection{Preliminaries: Constrained Polynomial Zonotopes}
\label{sec:prelim}

Throughout, $\sX \subset \R^n$ denotes a compact ambient workspace,
typically a bounded operational region of the end-effector or object pose.
For sets $A, B \subseteq \sX$, $A \cap B$, $A \cup B$, and
$A^c := \sX \setminus A$ denote intersection, union, and complement in
$\sX$. The (symmetric) Hausdorff distance is
$\Hdist{A}{B} := \max\big\{ \sup_{a \in A} \inf_{b \in B} \|a - b\|,~
                            \sup_{b \in B} \inf_{a \in A} \|a - b\| \big\}.$

\begin{wrapfigure}{r}{0.28\linewidth}
\vspace{-\intextsep}
\centering
\begin{tikzpicture}[scale=1.1, every node/.style={font=\scriptsize},
                     inner sep=0pt, outer sep=0pt]
  \filldraw[fill=blue!10, draw=blue!65, thick]
    (0, 0) rectangle (3.0, 2.4);
  \filldraw[fill=green!18, draw=green!55!black, thick]
    (0.6, 1.7) -- (1.5, 2.0) -- (2.4, 1.7)
    -- (2.4, 0.7) -- (1.5, 0.4) -- (0.6, 0.7) -- cycle;
  \filldraw[fill=red!22, draw=red!75, thick, line join=round]
    (1.2, 0.66)
    .. controls (0.55, 1.73) and (2.45, 1.73) .. (2.0, 0.66)
    .. controls (1.85, 1.08) and (1.35, 1.08) .. cycle;
  \node[blue!75] at (0.7, 0.2) {AABB};             
  \node[green!45!black] at (1.5, 1.61) {Zonotope}; 
  \node[red!75] at (1.5, 1.23) {CPZ};              
\end{tikzpicture}
\caption{Tightness hierarchy: axis-aligned
$\text{AABB} \supseteq$ centrally symmetric
$\text{Zonotope} \supseteq$ possibly non-convex $\text{CPZ}$.}
\label{fig:cpz_intuition}
\end{wrapfigure}
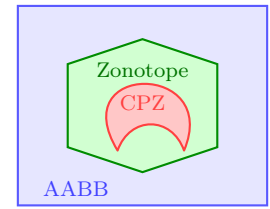

A constrained polynomial zonotope
(CPZ, \Cref{fig:cpz_intuition})~\citep{kochdumper2023constrained,zhang2025data,zhang2026exact}
$\sZ \subseteq \R^n$
is the image of a polynomial map in factor variables $\alpha \in [-1, 1]^p$
subject to polynomial equality constraints; written
$\sZ = \cpz(c, G, E, A, b, F)$ with center $c \in \R^n$, generator
matrix $G$, exponent matrix $E$, constraint generators $A$, constraint
exponent matrix $F$, and offset $b$. The formal definition and the
three CPZ operations we use are stated in
\Cref{app:cpz_def}: (i)~exact intersection $\sZ_1 \cap \sZ_2$ via
concatenation of constraint generators; (ii)~union
$\sZ_1 \cup \sZ_2$, representable as a CPZ via a binary indicator
factor (\Cref{app:union}); and (iii)~complement $\sX \setminus \sZ$,
exact when $\sZ$ is axis-aligned and approximate otherwise
(\Cref{app:complement}). Each operation costs polynomial time in the
representation size of the operands. These three operations realise $(\opand, \opor, \opnot)$ on language as $(\cap, \cup, ^c)$
on CPZs, the algebraic backbone of the rest of this section.

\subsection{Compositional Instruction Grammar}
\label{sec:grammar}

The grammar $\sL$ closes a fixed primitive set $\Phi$ under
$\{\opand, \opor, \opnot\}$ (i.e.\
$L ::= \phi \mid L_1 \opand L_2 \mid L_1 \opor L_2 \mid \opnot L_1$).
Primitives
are templated natural-language predicates with fixed geometric
semantics, in four families: goal (e.g.\ ``near $\langle o \rangle$''),
obstacle (e.g.\ ``avoid $\langle o \rangle$''), property
(e.g.\ ``keep $\langle p \rangle$ level''), and relational
(e.g.\ ``between $\langle o_1 \rangle$ and $\langle o_2 \rangle$'').
A VLM queried with chain-of-thought
prompting~\citep{wei2022chain} parses free-form manipulation
instructions into this grammar (\Cref{app:vlm}). The grammar
extends naturally to comparative spatial primitives such as
``closer to $\langle o_1 \rangle$ than $\langle o_2 \rangle$''
(a half-space $\{x : \|x-o_1\| \leq \|x-o_2\|\}$ supported
natively by the CPZ linear-constraint matrix); the multi-object
generalisation ``closest to $o_1$ among $\{o_1, \ldots, o_n\}$''
compiles to the intersection
$\bigcap_{i \geq 2} \{x : \|x-o_1\| \leq \|x-o_i\|\}$, i.e.\
the Voronoi cell of $o_1$ under the conjunction $\opand$.
Extensions to temporal logic~\citep{maler2004monitoring} and
counting are future work.

\subsection{Two-Sided CPZ Extraction}
\label{sec:extraction}

The extraction pipeline (VLM parsing $\to$ per-primitive
template selection with object detection $\to$ conformal margin
calibration $\to$ two-sided pair) is visualised in
\Cref{fig:pipeline} (full implementation in \Cref{app:impl}).
For each primitive $\phi \in \Phi$ appearing in an instruction $L$, we
extract a two-sided CPZ enclosure $(\hat \sC_\phi^-, \hat \sC_\phi^+)$
of the true feasible set $\sC_\phi \subseteq \sX$ such that, with high
probability, $\hat \sC_\phi^- \subseteq \sC_\phi \subseteq \hat \sC_\phi^+$;
a pair rather than a single set preserves containment under
arbitrary $\opand/\opor/\opnot$ composition.

Per-primitive templates $T_\phi$ (e.g. \primfam{near} as a centroid
box, \primfam{between} as a half-space intersection, \primfam{avoid}
as an ellipsoid complement) are thin wrappers over standard CPZ
constructors (\Cref{app:impl}), and the conformal margin
$\epsilon_\mathrm{cal}$ follows the standard split-conformal
recipe~\citep{vovk2005algorithmic,angelopoulos2023gentle}; the
algebra is invariant to swapping any learning component
(parser, detector, calibration), decoupling sound composition
from learned perception. The analysis rests on a single
distributional assumption, per-primitive coverage
(\Cref{ass:perPrimitive}): each extracted pair contains the true
clause-feasible set with probability at least $1 - \delta_\phi$,
enforced by conformal prediction during extraction. A mild
operator-regularity condition bounds how endpoint perturbations
propagate through $\opand, \opor$; on axis-aligned grammars the
Lipschitz constant is $\kappa = 1$ (formal statements:
\Cref{app:proofs}).

\subsection{CPZ Composition Algebra}
\label{sec:composition}

Given per-primitive enclosures, we compose $L \in \sL$ bottom-up
into one CPZ enclosure pair $(\hat \sC_L^-, \hat \sC_L^+)$
(\Cref{eq:composition} in \Cref{app:union}): conjunction takes
componentwise intersection, disjunction takes componentwise union,
and negation swaps and complements,
$\opnot L_1 \mapsto \big((\hat \sC_{L_1}^+)^c,\, (\hat \sC_{L_1}^-)^c\big)$
to preserve containment. The three set operations $\cap$, $\cup$, and workspace complement
are realised by standard CPZ primitives (\Cref{app:cpz_def}). Size
and time complexity: \Cref{prop:complexity}; pseudocode:
\Cref{alg:compose} in \Cref{app:impl}. Theoretical guarantees
follow in \Cref{sec:theory}.

\subsection{Theoretical Guarantees}
\label{sec:theory}

Under per-primitive conformal coverage (\Cref{ass:perPrimitive})
as the sole distributional assumption, the composed enclosure
satisfies a union bound
$\Prob[\hat \sC_L^- \subseteq \sF(L) \subseteq \hat \sC_L^+]
\geq 1 - \sum_i \delta_i$~--- to our knowledge the first
distribution-free coverage guarantee for VLM-grounded constraint
extraction under compositional language. The Hausdorff error
against the true feasibility set is bounded by
$\kappa^{\mathrm{depth}(L)} \epsilon$ in the per-primitive
thickness $\epsilon$; under
$\epsilon < \eta/(2\kappa^{\mathrm{depth}(L)})$ CompCPZ recovers
all $k$ semantic modes, while any convex over-approximation
inherits the $\eta/2$ floor of \Cref{thm:convexLowerBound}. On
axis-aligned primitives $\kappa \leq 1$, so the bound is
depth-independent; empirically $\hat\kappa = 1.00$ over $6{,}828$
internal parse-tree nodes (\Cref{fig:empirical_kappa}) and even
over $1{,}880$ rotated-box nodes outside the axis-aligned
hypothesis. Composition runs in linear time in $|L|$ with $O(1)$
containment queries, and inference reduces to LP feasibility;
full statements and proofs: \Cref{app:proofs}.

\paragraph{Linguistic compositionality is a topological invariant.}
The disjunctive depth of $L$ lower-bounds the component count of
its feasibility set: $\beta_0(\sF(L)) \geq \beta_0(\sF(L_1)) +
\beta_0(\sF(L_2))$ across an $\opor$ between propositionally
inconsistent operands. Write $\mathcal{C}_{\mathrm{top}}(L) :=
\beta_0(\sF(L))$ for the topological complexity of $L$ and $\eta$
for the inter-mode separation of its $k = \mathcal{C}_{\mathrm{top}}(L)$
modes.

\noindent\textbf{Theorem (Topological lower bound, informal).}
\emph{Any sound encoding $S \supseteq \sF(L)$ with fewer than $k$
connected components satisfies $\Hdist{S}{\sF(L)} \geq \eta/2$.}
Formal statement and proof in \Cref{app:proofTopological}
(\Cref{thm:topologicalLowerBound}, workspace-free, with convex
corollary \Cref{thm:convexLowerBound}). The bound rules out
every sound single-component encoding; non-sound outputs
(multi-peak value maps, heatmap argmaxes, single-point actions,
diffusion samples) fall outside its scope and are governed by
\Cref{prop:singlePoint} once they commit to a single executed
action. CompCPZ attains the bound via its $\mathrm{CPZUnion}$
representation; an optional mode-merge pass (\Cref{app:union})
restores strict component minimality if the lazy implementation
over-splits.

\paragraph{A single action cannot certify a $k$-mode disjunction.}

\noindent\textbf{Proposition (Single-action lower bound, informal).}
\emph{For a $k$-mode instruction with distribution $\mu$ over
run-time mode availability, any planner committing to a single
$a \in \R^n$ before observing $\mu$ has worst-case success rate at
most $1/k$ under symmetric $\mu$, and $0$ under adversarial $\mu(a)$.}
Formal statement and proof in \Cref{app:proofSinglePoint}
(\Cref{prop:singlePoint}).

\section{Experiments}
\label{sec:exp}

Existing manipulation benchmarks (LIBERO~\citep{liu2024libero},
ManiSkill3~\citep{tao2025maniskill3}) test single-target
instructions, where $\mathcal{C}_{\mathrm{top}}(L) = 1$ and
\Cref{thm:topologicalLowerBound} is degenerate. We therefore
construct a $49$-instruction (extended: $191$) tabletop
benchmark stratified by depth $\in \{0,\ldots,4\}$ and primitive
family (\primfam{near}, \primfam{inside}, \primfam{between},
\primfam{avoid}) over $11$ unmodified ManiSkill3 task families
(\Cref{app:benchmark}). To rule out a benchmark-construction
artefact, we retain the single-target sanity tasks PickCube and
StackCube on which CompCPZ does not beat the convex baseline
(both $\geq 0.99$ in-GT-mode, \Cref{tab:closed_loop}); zero-shot
OpenVLA-7B~\citep{kim2024openvla} on the same task families
collapses to $30.8\%$ in-GT-mode (\Cref{app:stat_tests}),
independent evidence the gap is task-intrinsic.

Baselines: AABB convex (union-of-modes axis-aligned hull), MVEE
convex (tightest convex hull), and an oracle GMM with $k$
Gaussian modes (\Cref{app:gmm})---all
instantiating the sound single-component class governed by
\Cref{thm:convexLowerBound}. Non-sound multi-peak methods
(VoxPoser~\citep{huang2023voxposer}, OWL-TAMP~\citep{owltamp2024},
PhysVLM~\citep{zhou2025physvlm}, RoboPoint~\citep{yuan2024robopoint})
are governed by \Cref{prop:singlePoint} or their decoded $\beta_0$;
why these cannot serve as sound-enclosure baselines, and why our
committed-action proxies are faithful, is detailed in
\Cref{app:extended_baselines}.

\subsection{Main Results}
\label{sec:exp_main}

\begin{figure}[t]
\centering
\includegraphics[width=0.245\linewidth]{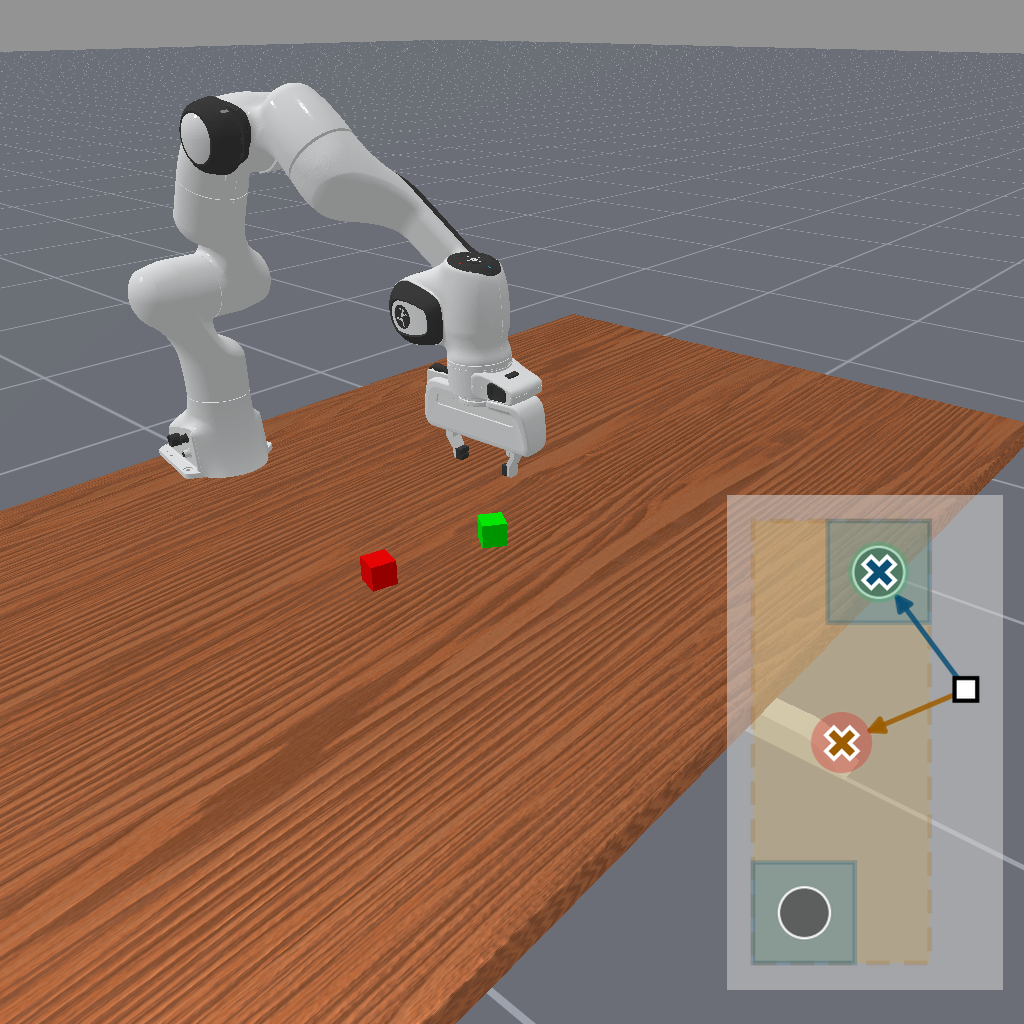}\hfill
\includegraphics[width=0.245\linewidth]{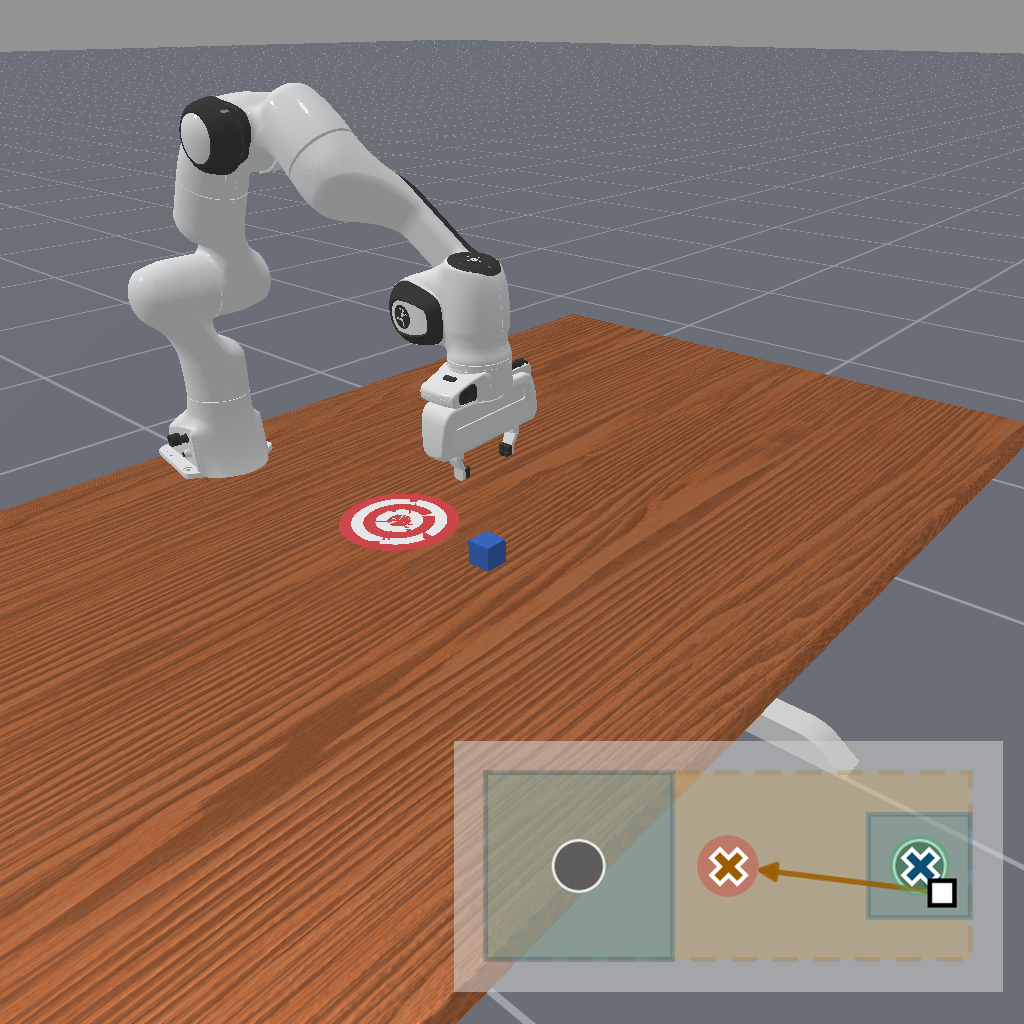}\hfill
\includegraphics[width=0.245\linewidth]{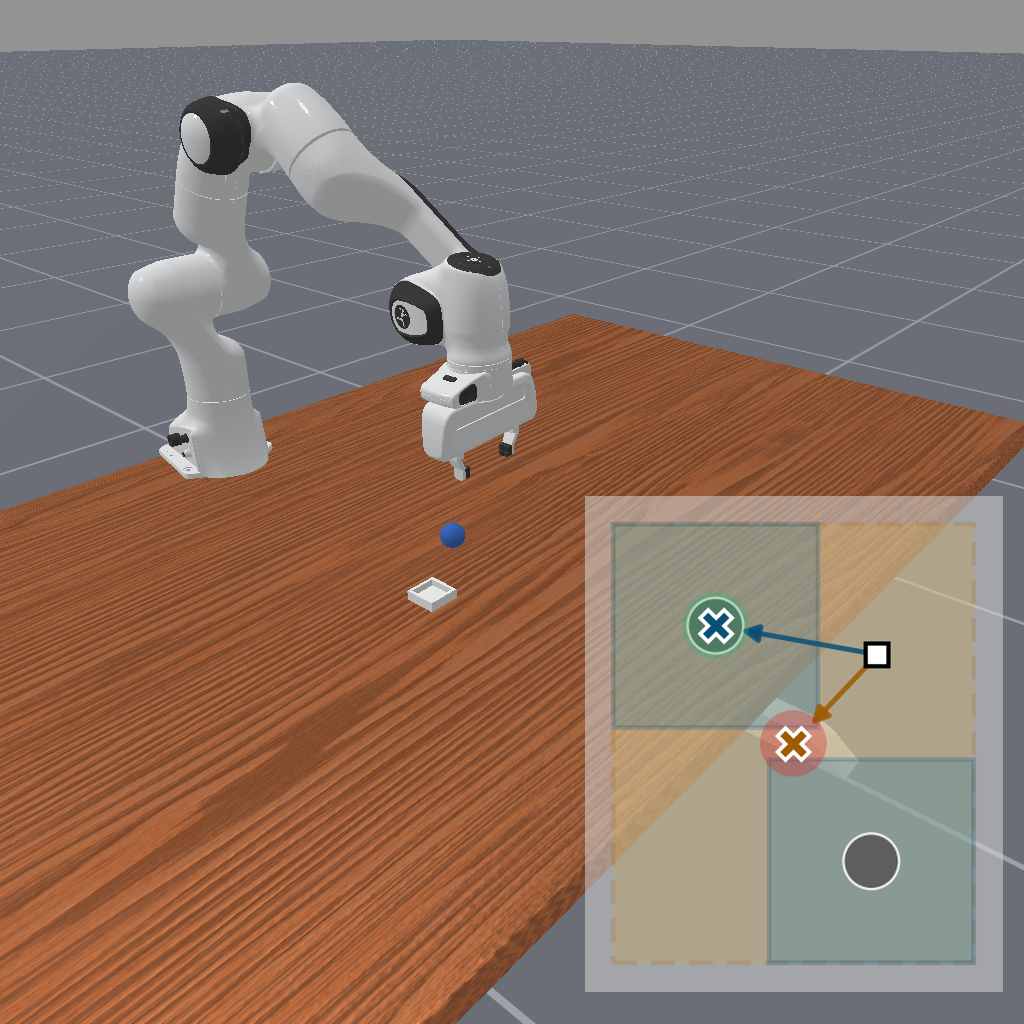}\hfill
\includegraphics[width=0.245\linewidth]{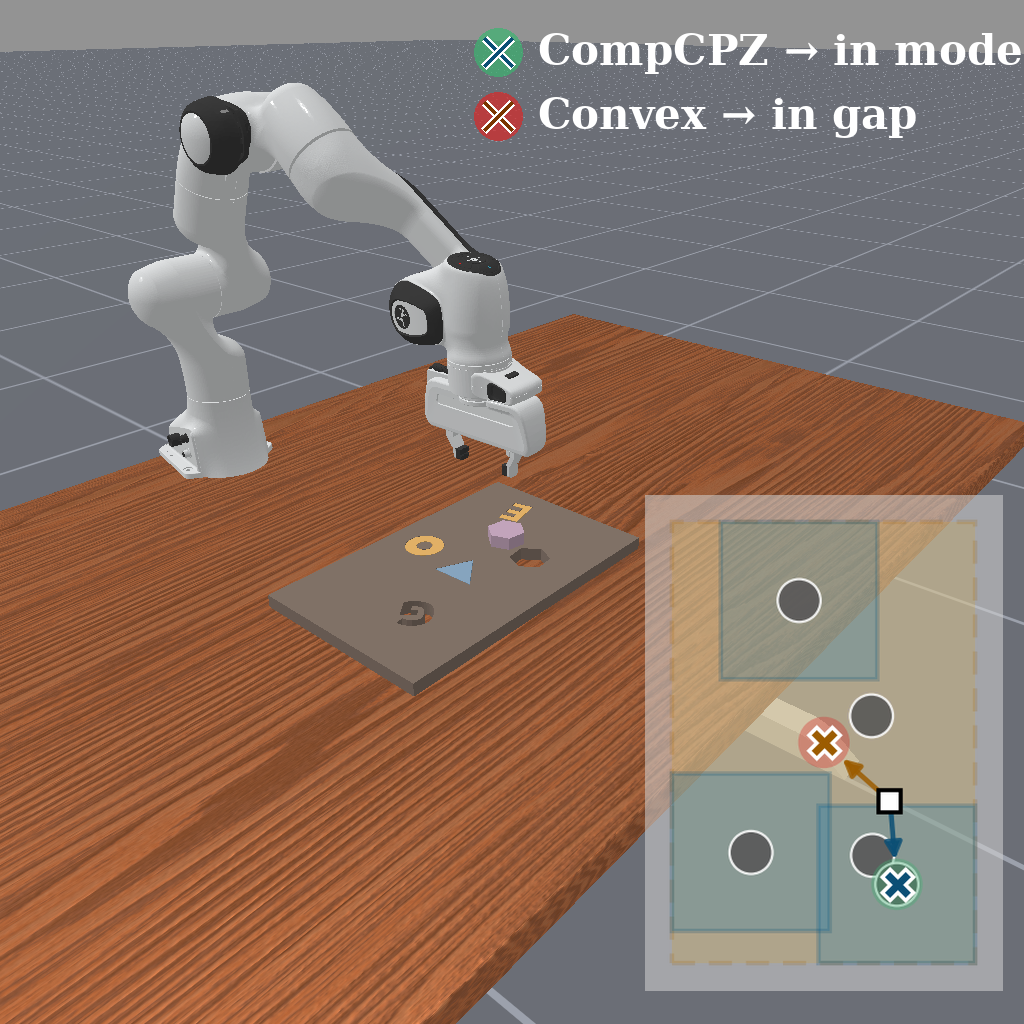}
\caption{Sound convex enclosures structurally fail on disjunctive
instructions (\Cref{thm:convexLowerBound}). Four ManiSkill3 task
families, each pairing a SAPIEN render with a top-down inset. The
CompCPZ target (green halo) lands inside a true mode; the convex
AABB target (red halo) lands in the inter-mode gap.}
\label{fig:teaser}
\end{figure}

The headline question is not whether the predicted set looks
tight, but whether using it as a goal oracle places the robot
inside a region that satisfies the instruction. Metrics: geometric
success (EE within $5$\,mm of target) and in-GT-mode
(non-self-referential: EE inside any analytic GT feasibility
region). Both methods run under the same Franka servo
(\Cref{sec:exp_closedloop}); \Cref{fig:teaser} shows four
representative families where the convex target lands in the
inter-mode gap while CompCPZ reaches a true mode.
The two non-self-referential metrics summarised over $18$ task
families ($\times\,200$ seeds, with mode-switch variants) are
reported in \Cref{tab:closed_loop}; the visual companion is
\Cref{fig:compcpz_bars}; the MVEE sharpening is
\Cref{fig:mvee_bars} ($\times\,20$ seeds, $p \ll 10^{-30}$ pooled;
exact values in \Cref{app:stat_tests}).

\begin{wraptable}{r}{0.40\linewidth}
\vspace{-\intextsep}
\centering
\footnotesize
\setlength{\tabcolsep}{4pt}
\renewcommand{\arraystretch}{0.95}
\caption{Synthetic ellipsoid-avoid (in-GT-mode, $n{=}459$): box
unions cannot tile an ellipsoid complement tightly at any box
count.}
\label{tab:ellipsoid_avoid}
\begin{tabular}{l c}
\toprule
Method & In-GT-mode \\
\midrule
CompCPZ                 & \textbf{1.000} \\
AABB-union (bounding)   & 0.928 \\
AABB-union (inscribed)  & 0.715 \\
\bottomrule
\end{tabular}
\end{wraptable}
The $1.00$ vs $0.13$ gap on clean disjunctions
(\Cref{tab:closed_loop}) is the structural gap between
single-component and any sound $k$-component encoder predicted
by \Cref{thm:topologicalLowerBound}. CompCPZ extends beyond
axis-aligned-positive primitives with four capabilities:
adversarial mode blocking (CompCPZ pooled $\mathbf{0.90}$
vs the $1{-}1/k$ floor a nearest-centroid union collapses to,
\Cref{tab:nonsound_adversarial}); explicit negation (a
complement-free union loses $\mathbf{14.75\%}$ of pooled trials
on the headline negation instructions, \Cref{tab:negation_ablation});
distribution-free conformal coverage propagation through Boolean
composition; and polynomial primitives such as ellipsoid
avoidance (\Cref{tab:ellipsoid_avoid}). All four are derived
compositionally from a natural-language parse tree rather than
enumerated by hand.

\paragraph{Polynomial primitives demand the CPZ tier.}
The clean-disjunction gap above lives at the AABB tier of the
tightness hierarchy (\Cref{fig:cpz_intuition}); the CPZ tier is
exercised by curved primitives, for which axis-aligned boxes
cannot be simultaneously sound and tight. On a synthetic
disjunctive-avoid benchmark (``near red $\opor$ near blue, avoid
the obstacle''; $500$ scenarios with random goal boxes and a
random ellipsoid obstacle, true feasibility
$(\mathrm{red}_\square \cup \mathrm{blue}_\square) \cap
\{x : (x{-}c)^\top \mathrm{diag}(a^{-2}, b^{-2})(x{-}c) \geq 1\}$),
CompCPZ encodes the polynomial obstacle exactly and reaches
$1.00$ in-GT-mode. A union-of-AABBs must replace the ellipsoid by
a box: the bounding-box variant loses goal corners outside the
ellipsoid ($0.928$), the inscribed-box variant selects targets
inside it ($0.715$; \Cref{tab:ellipsoid_avoid}).

\subsection{Closed-Loop Franka Execution}
\label{sec:exp_closedloop}

A damped-pseudoinverse Jacobian servo drives the Franka EE toward each
method's target (CompCPZ: closest mode centroid; AABB baseline: centroid
of the bounding box of all modes). The ground-truth feasibility
box is computed by evaluating the parse tree against live
ManiSkill3 actor positions.

\begin{figure}[t]
\centering
\includegraphics[width=\linewidth]{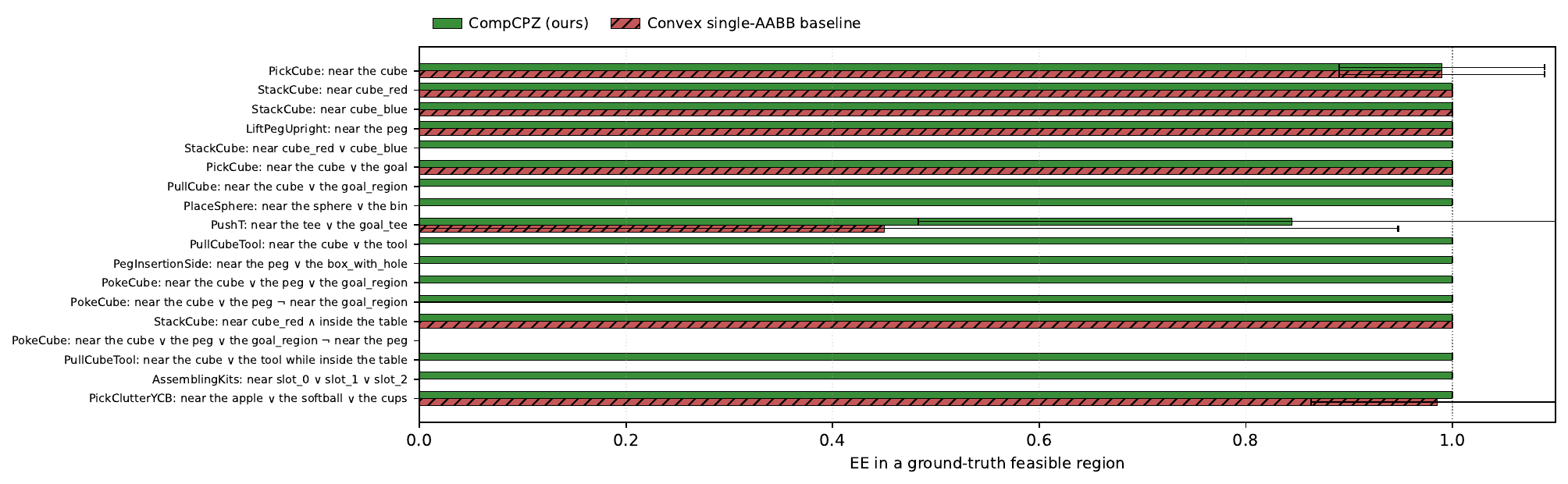}
\caption{Closed-loop Franka EE servo on $18$ ManiSkill3 task
families ($200$ seeds each). Both methods reach the target at
$\geq\!0.99$ geometric success; only CompCPZ lands the EE inside
an analytic GT mode. Convex silently fails on $10/18$ multi-modal
families (red ``silent failure'' annotations). Pooled paired sign
test: $1{,}900/1{,}918$ wins, $p \ll 10^{-30}$
(\Cref{app:stat_tests}).}
\label{fig:compcpz_bars}
\end{figure}

\begin{table}[t]
\caption{Closed-loop in-GT-mode rate on ManiSkill3 ($200$ seeds
per row). Both methods reach $\geq\!0.99$ geometric success; only
CompCPZ lands the EE inside a GT mode. Pooled paired sign test:
$1{,}900/1{,}918$ wins, $p \ll 10^{-30}$ (\Cref{app:stat_tests}).
PushT: Franka kinematic edge.}
\label{tab:closed_loop}
\centering
\footnotesize
\renewcommand{\arraystretch}{0.93}
\setlength{\tabcolsep}{3pt}
\begin{tabular}{l l c | c c}
\toprule
Task & Instruction & $k$ & CompCPZ & Convex \\
\midrule
PickCube         & near cube                              & 1 & 0.99 & 0.99 \\
StackCube        & red $\vee$ blue                        & 2 & \textbf{1.00} & \textbf{0.00} \\
PokeCube         & cube $\vee$ peg $\vee$ goal            & 3 & \textbf{1.00} & \textbf{0.00} \\
PokeCube         & (cube $\vee$ peg) $\wedge \neg$ goal   & 2 & \textbf{1.00} & \textbf{0.00} \\
AssemblingKits   & slot\_0 $\vee$ slot\_1 $\vee$ slot\_2  & 3 & \textbf{1.00} & \textbf{0.00} \\
PickClutterYCB   & apple $\vee$ softball $\vee$ cups      & 3 & 1.00 & 0.98 \\
PushT            & tee $\vee$ goal\_tee                   & 2 & 0.89 & 0.51 \\
\midrule
\multicolumn{3}{l|}{\textbf{Pooled $18$ families} ($10$ more in App.~\ref{app:stat_tests})} & \textbf{1.00} & \textbf{0.13} \\
\bottomrule
\end{tabular}
\end{table}
\paragraph{Baseline ablations.}
The tighter convex MVEE baseline fails by the
same convex lower-bound argument, and two non-sound
multi-peak decoders (heatmap argmax in PhysVLM/RoboPoint style and
a VoxPoser-lite sampling decoder) collapse to the $1 - 1/k$ floor
of \Cref{prop:singlePoint}; CompCPZ wins on the same data
(\Cref{tab:nonsound_adversarial}, \Cref{app:extended_baselines}).
The four differentiator regimes are quantified separately. Under
adversarial mode blocking, a
hand-rolled union committing to the nearest centroid falls to the
$1-1/k$ floor while CompCPZ retains pooled mean $\mathbf{0.90}$
(vs heatmap argmax $0.54$, sampling decoder $0.52$) across $13$
multi-modal families via containment-checked mode switching. On
the two negation instructions, an AABB-union without complement
loses $\mathbf{14.75\%}$ of pooled trials
(\Cref{tab:negation_ablation}). Zero-shot
OpenVLA-7B~\citep{kim2024openvla}, closed-loop on all $13$
multi-modal families ($\times 50$ seeds, $650$ episodes),
achieves $30.8\%$ in-GT-mode ($4/13$ cells at $1.00$,
$9/13$ at $0.00$) and task-success $0\%$. OpenVLA emits a single
$7$-DoF action rather than a sound enclosure, so the
set-representation bound does not apply; the $69.2\%$
gap-landings are however consistent with
\Cref{prop:singlePoint}. Pooled across the $18$-task, $200$-seed
evaluation ($n=1{,}918$ decisive trials), the one-sided paired
sign test returns $1{,}900$ CompCPZ wins vs $18$ AABB losses,
$p \ll 10^{-30}$ against every baseline; full tables in
\Cref{app:stat_tests}.

\paragraph{Headline numbers transfer to a real VLM parser.}
The headline closed-loop numbers above use a deterministic regex
parser to isolate the geometric origin of the gap from
natural-language parsing noise. To confirm transfer to an actual
VLM, we re-ran the full $18$-task $\times$ $200$-seed closed-loop
benchmark with GPT-4o as the live parser. The pooled in-GT-mode
rate is identical at the trial level: CompCPZ $3{,}367/3{,}600$
in-GT-mode under both regex and GPT-4o parsers (matching across
every one of the $18$ task-instruction cells, per-cell
$\Delta = 0$); the $1{,}900/1{,}918$ paired-sign-test wins and
$p \ll 10^{-30}$ reported above therefore hold verbatim under
the live VLM parser.

\subsection{Cross-Embodiment Hardware Validation: Unitree Go2}
\label{sec:exp_go2}

\begin{wraptable}{r}{0.60\linewidth}
\vspace{-\intextsep}
\centering
\footnotesize
\caption{Real-robot Go2 in-GT-mode rate ($36$ trials, $3$ seeds per
cell). All methods reach their planner goal in $12/12$ trials; the
gap is whether that goal is inside a GT mode. S3 excludes the rug;
S4 blocks mode~$0$ (CompCPZ informed). Time/travel:
\Cref{app:go2_results}.}
\label{tab:go2_results}
\setlength{\tabcolsep}{3pt}
\renewcommand{\arraystretch}{1.05}
\begin{tabular}{l c c c c c}
\toprule
Method & S1 & S2 & S3 ($\neg$rug) & S4 ($\neg$m$_0$) & Total \\
\midrule
CompCPZ       & \textbf{3/3} & \textbf{3/3} & \textbf{3/3} & \textbf{3/3} & \textbf{12/12} \\
AABB convex   & 0/3 & 0/3 & 0/3 (on rug) & 0/3 & \textbf{0/12} \\
Single clause & 2/3 & 2/3 & 3/3 & \textbf{0/3 (blocked)} & 7/12 \\
\bottomrule
\end{tabular}
\end{wraptable}
To probe whether the same compiler transfers across embodiments
and runs unchanged on physical hardware, we deploy the identical
grammar and CPZ algebra on a Unitree Go2
quadruped with
planar configuration $\sX \subset \R^2$, parsed by an LLM
(Claude / GPT-4) and executed under $\sim\!100$\,Hz NOKOV motion
capture. The four scenarios cover the $2$-mode disjunction (S1),
the $3$-mode disjunction (S2), the $2$-mode disjunction with rug
negation (S3), and the adversarial mode-block variant of S1 (S4);
three methods (CompCPZ, AABB convex, single-clause) $\times$
three seeds give $36$ closed-loop trials (\Cref{fig:go2_real}).
Go2 (a legged, non-anthropomorphic platform with a different
pursuit controller and state estimator) is a strictly larger
embodiment gap than a same-embodiment Franka sim$\to$real swap;
the compiler runs unchanged across both, isolating set-grounding
behaviour from embodiment bias. Direct Franka-class manipulation
(end-effector reachability, contact dynamics) is complementary
follow-up work. Hardware stack: \Cref{app:go2}; per-trial table:
\Cref{tab:go2_results}.

\begin{figure}[t]
\centering
\includegraphics[width=0.245\linewidth]{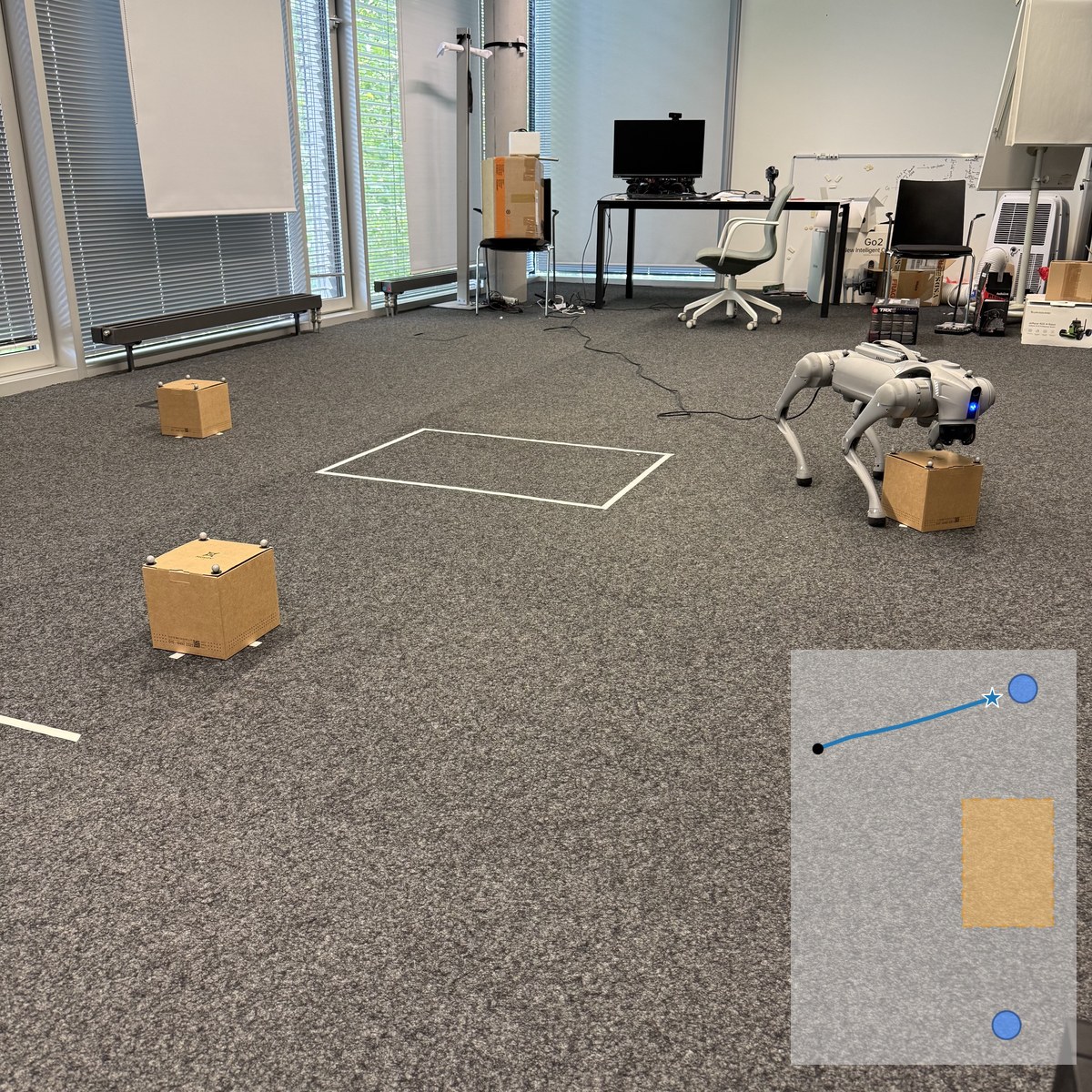}\hfill
\includegraphics[width=0.245\linewidth]{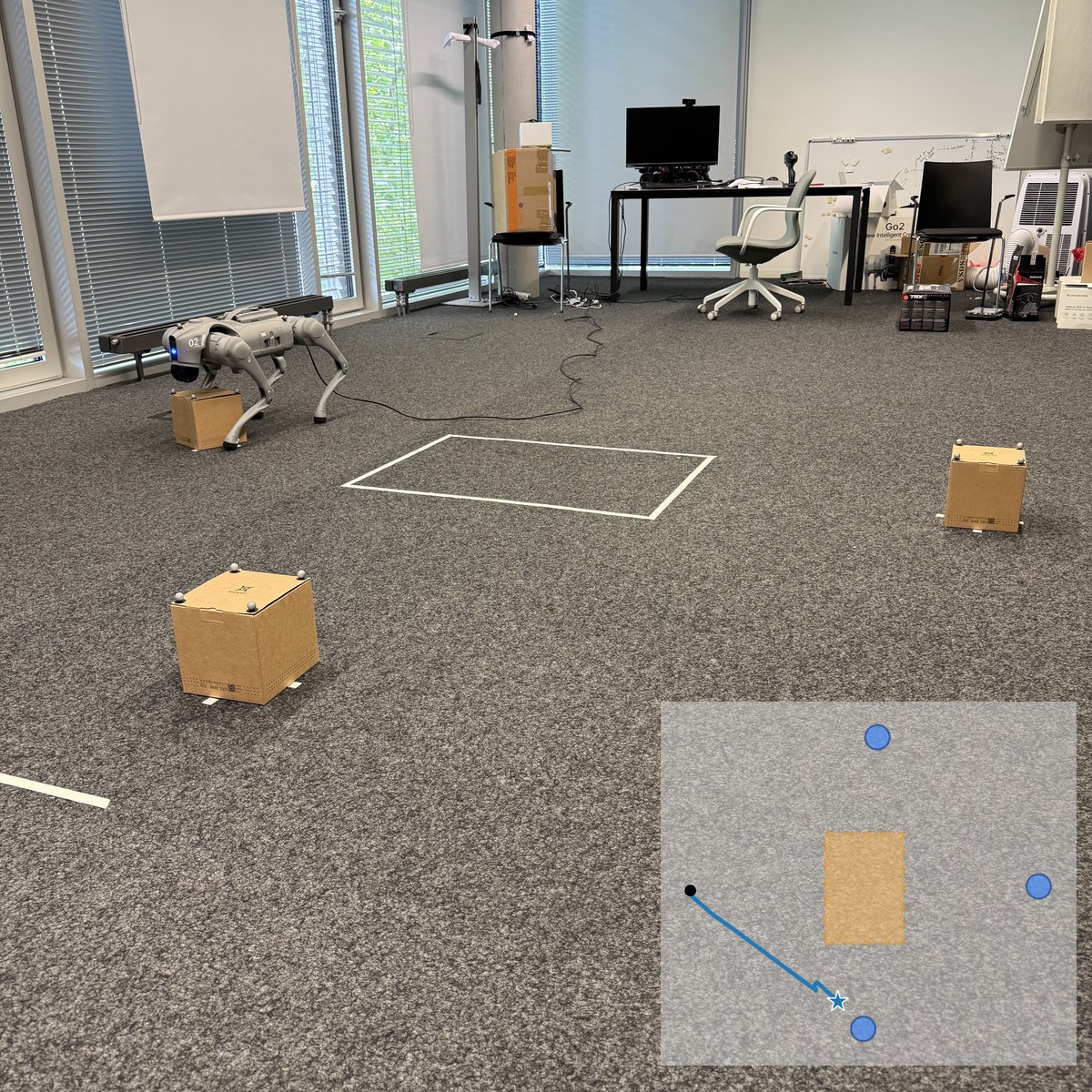}\hfill
\includegraphics[width=0.245\linewidth]{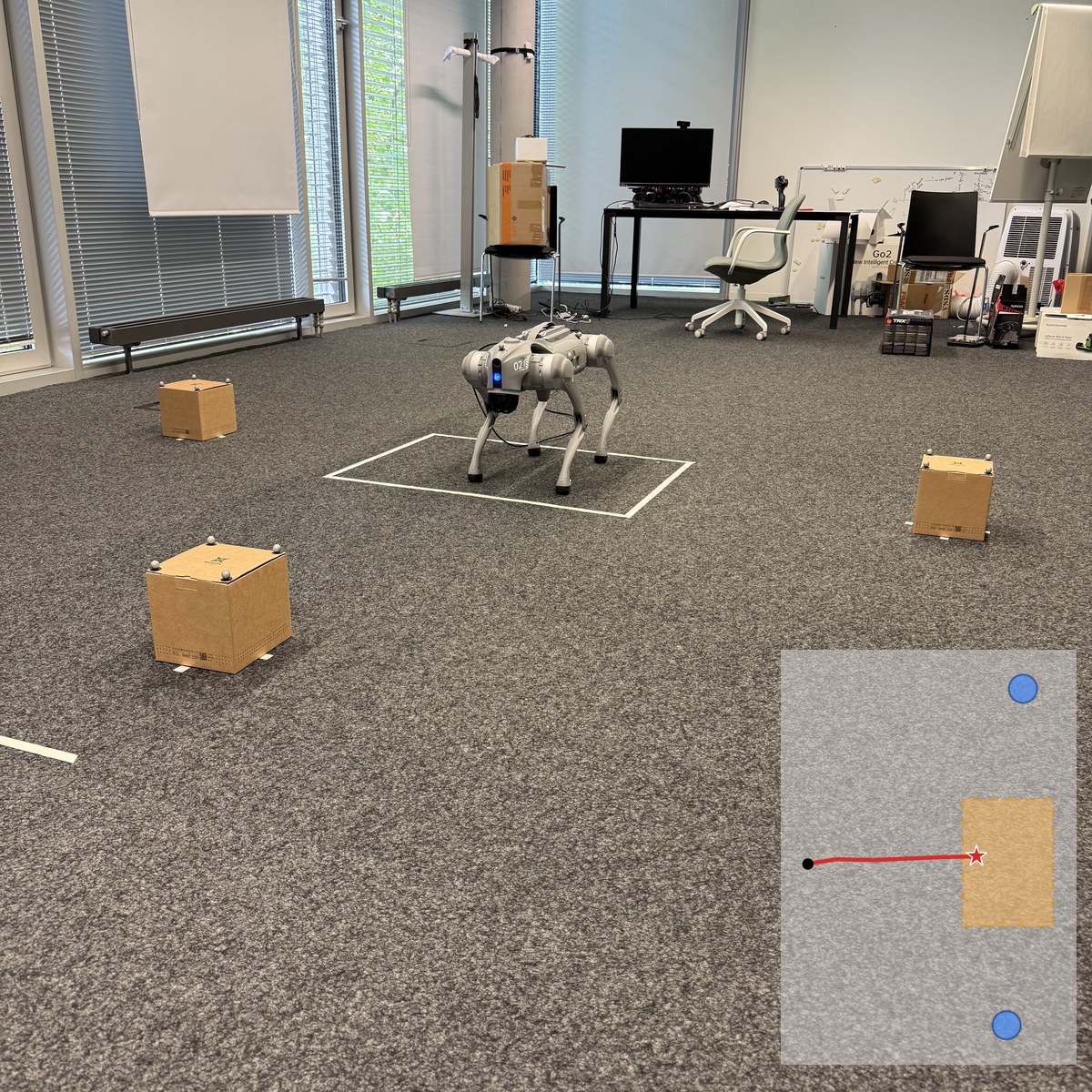}\hfill
\includegraphics[width=0.245\linewidth]{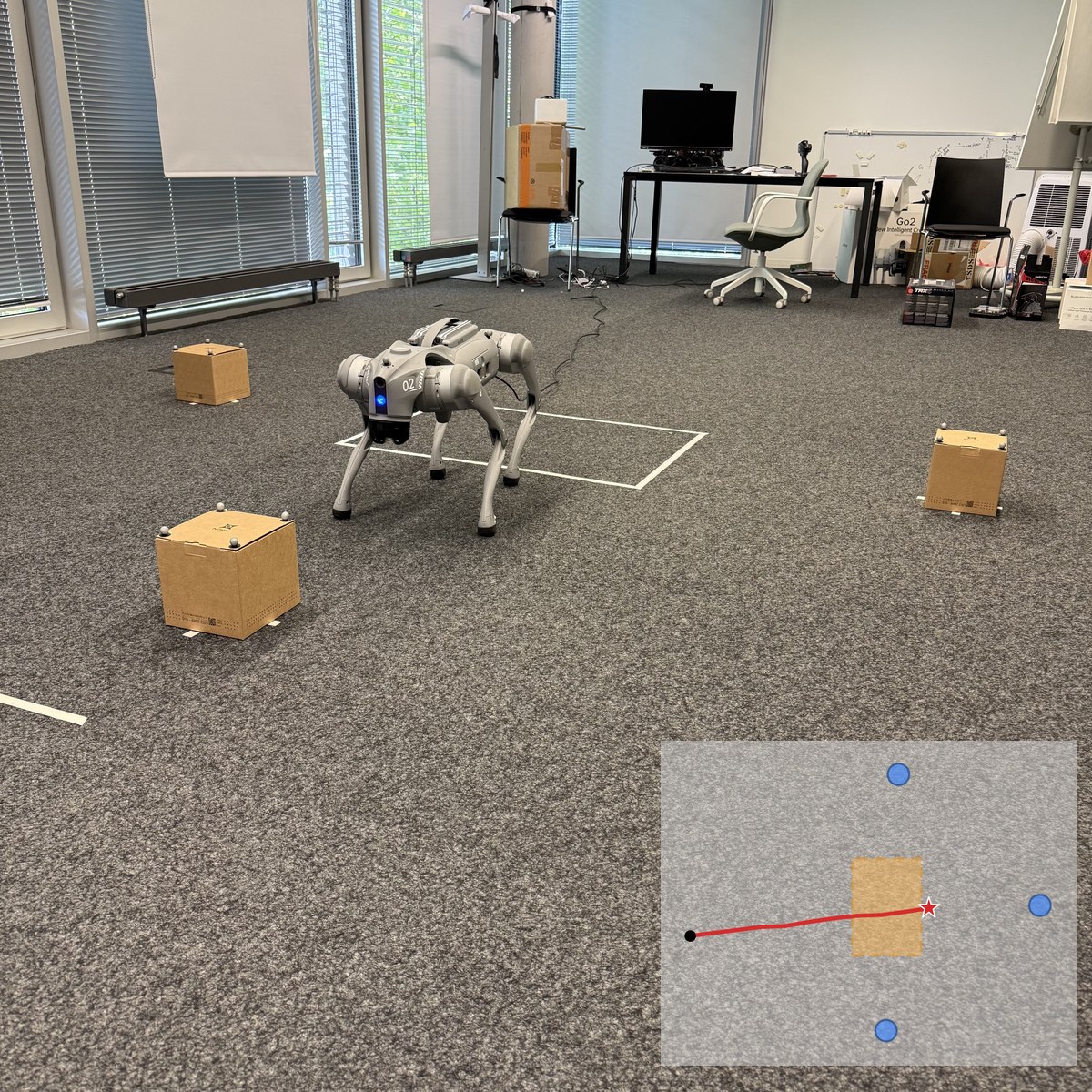}
\caption{Real-robot Go2 trials ($36$ trials total). Each panel:
end-state photo with top-down MoCap trajectory inset (CompCPZ
blue, AABB red). \textbf{(a-b)} CompCPZ reaches a GT mode.
\textbf{(c)} AABB lands the dog on the forbidden rug (negation
violation). \textbf{(d)} AABB stops in the inter-mode gap. Full
per-trial table: \Cref{tab:go2_results}.}
\label{fig:go2_real}
\end{figure}

\paragraph{The structural failure transfers to hardware.}
CompCPZ lands the dog strictly inside an analytic
ground-truth mode $\sC_\phi^{(k)}$ in $9/9$ trials over the
disjunction scenarios (S1, S2) and the negation scenario (S3), and
$3/3$ on the adversarial mode-block scenario (S4)---a pooled
$\mathbf{12/12}$ correct-mode rate on a hardware platform the
compiler has never been tuned for. The convex-hull baseline lands
inside \emph{no} true mode in $\mathbf{0/12}$ trials, and on S3
lands the dog inside the negated rug region in $\mathbf{3/3}$ runs
(a hard semantic violation, visible in the supplementary video).
The single-clause baseline, which commits to the first disjunct
before observing run-time mode availability, lands inside a true
mode on $7/9$ unblocked-scenario trials (S1, S2, S3) but on the
adversarial mode-block scenario S4 walks into the blocked mode in
$\mathbf{0/3}$ trials, the real-hardware analogue of
\Cref{prop:singlePoint}. S3 and S4 are the regimes a hand-rolled
union-of-AABBs would also fail on (no exact complement, no
informed mode switching); the hardware therefore tests the
CompCPZ-specific value-add directly, not just the
single-component-vs-$k$-component gap.

\subsection{Runtime and Ablations}
\label{sec:exp_ablation}

Composition is $<\!0.25$\,ms at depth $4$, $<\!1.1$\,ms at the
worst-case disjunctive $d{=}7$ (\Cref{fig:depth_scaling}, matching the
$O(2^d)$ bound of \Cref{prop:complexity}). Empirical $\hat\kappa = 1$
across $6\,828$ nodes (\Cref{fig:empirical_kappa}). Conformal
coverage tracks $1{-}\alpha$ within
$\pm 1\%$ under Gaussian, heavy-tailed, anisotropic, and bias-mixture
noise (\Cref{fig:conformal_robust}). A $5\times5$ servo
hyperparameter grid yields $1.00$ success on all $75$ runs. A separate
parser-depth stress test on $49$ curated instructions of
increasing parse depth gives mode-count equivalence on all $49$
($49/49 = 100\%$; \Cref{tab:parser_depth}); the $16$ tree-shape
rearrangements GPT-4o introduces are all semantic-equivalent
rewrites (commutativity / associativity / De Morgan) to which the
compile rule \Cref{eq:composition} is invariant.

\paragraph{Perception envelope and data-efficient adaptation.}
A calibrated centroid-noise sweep (\Cref{tab:perception_noise},
\Cref{app:perception_envelope}) shows CompCPZ holds in-mode at
$1.00$ through $\sigma = 10$\,mm ($4\times$ the cube half-extent),
degrades to $0.87$ at $\sigma = 20$\,mm, and collapses by
$\sigma = 50$\,mm; the non-sound multi-peak baselines mirror this
curve, so the noise bottleneck is the geometric overlap of mode
enclosures, not the choice of decoder. To bring perception inside
the working envelope we contribute a data-efficient
deployment-time adaptation step: a $1$-GPU-hour in-domain
fine-tune of YOLOv8n on $800$ synthetic frames reaches median
centroid error $0.73$\,cm at $0\%$ miss, one to two orders of
magnitude below the budgets routinely used by VLA / RL policies
for sub-centimetre precision; a data-efficiency sweep
(\Cref{tab:training_data_ablation}) shows the envelope is already
cleared at $200$ frames and $\sim\!3$\,min of training. Off-the-shelf
open-vocab detectors on ManiSkill3 RGB sit well above this envelope
(OWL-ViTv2-base: $12.2$\,cm; Grounding-DINO+SAM2: $15.8$\,cm;
YOLO-World v2-s: $1.00$ miss), so a learning-based perception layer
is required for sub-centimetre deployment regardless of the
downstream algebra. The CompCPZ algebra itself is unchanged across
these detector swaps---the algebra and the learned perception layer
are orthogonal deployment levers. An oracle GMM with
ground-truth mode count, centroids, and calibrated covariance
attains $90.2\%$ in-mode vs.\ CompCPZ's $100\%$ on the
$31$-instruction multi-modal subset ($6\,200$ trials), reflecting
the commit-then-execute pathology of \Cref{prop:singlePoint}.
Full tables: \Cref{app:gmm_oracle}.

\section{Conclusion}
\label{sec:conclusion}

We recast the silent failure of vision-language pipelines on
disjunctive manipulation instructions as a structural property of
the represented feasibility set: no sound constraint extractor that
returns a single connected region can cover all modes of a
multi-modal instruction, and no single-action planner can avoid
worst-case degradation as the mode count grows. To meet both
bounds we introduce CompCPZ, a compositional compiler from natural
language to set-algebraic constraints with distribution-free
coverage propagation. Across a large simulated manipulation
benchmark and planar real-robot trials, CompCPZ outperforms convex
sound baselines, an oracle Gaussian-mixture baseline, non-sound
multi-peak decoders, and a zero-shot vision-language-action model.
We hope this work shifts how compositional language grounding
is evaluated: from decoded-set volume or single-action success
toward the connected-component structure of the represented
feasibility set, the topological invariant to which silent
multi-modal failures reduce.

\section{Limitations}
\label{sec:limitations}

CompCPZ has three scoping limitations. (i)~Depth-independent
stability is proved for axis-aligned grammars; rotated primitives
inherit a depth-dependent bound. (ii)~It is a single-shot
grounding oracle; long-horizon and temporal-logic extensions are
future work. (iii)~Real-robot validation
uses a planar quadruped; contact-rich manipulation hardware is
complementary future work. The topological lower bound covers
sound encoders; non-sound systems are addressed empirically.

\clearpage
\bibliographystyle{plainnat}
\bibliography{references}

\clearpage
\appendix

\section{Formal CPZ Definition and Operations}
\label{app:cpz_def}

\begin{definition}[Constrained Polynomial Zonotope]
\label{def:cpz}
A constrained polynomial zonotope (CPZ) $\sZ \subseteq \R^n$ is
the image of a polynomial map in factor variables
$\alpha \in [-1, 1]^p$ subject to equality constraints:
\begin{equation}
  \sZ = \Big\{ c + \sum_{i=1}^h g_i \!\!\prod_{j=1}^{p} \alpha_j^{E_{ji}}
              ~\Big|~ \alpha \in [-1,1]^p,~
              \sum_{i=1}^{q} a_i \!\!\prod_{j=1}^{p} \alpha_j^{F_{ji}} = b
        \Big\},
\end{equation}
parametrised by center $c \in \R^n$, generator matrix $G = [g_1,\ldots,g_h]$,
exponent matrix $E \in \N^{p \times h}$, constraint generators
$A = [a_1,\ldots,a_q]$, constraint exponent matrix
$F \in \N^{p \times q}$, and offset $b$.
We write $\sZ = \cpz(c, G, E, A, b, F)$.
CPZs subsume zonotopes (no constraints, exponents $\in\{0,1\}$),
constrained zonotopes (no nonlinear exponents), and polynomial zonotopes
(no constraints) as special cases.
\end{definition}

\begin{proposition}[CPZ operations~\citep{kochdumper2023constrained,zhang2025data}]
\label{prop:cpzOps}
The class of CPZs is closed under
{\rm (i)}~exact intersection $\sZ_1 \cap \sZ_2$ via concatenation of constraint
generators;
{\rm (ii)}~union $\sZ_1 \cup \sZ_2$, representable as a CPZ on
$\sX$ via a binary indicator factor (\Cref{app:union}); and
{\rm (iii)}~complement in a bounded box $\sX \setminus \sZ$, exactly when
$\sZ$ is a polytope (zonotope or constrained zonotope) and approximately
otherwise (\Cref{app:complement}).
The cost of each operation is polynomial in the representation size of the
operands; intersection adds $q_1 + q_2$ constraints, union adds one factor.
\end{proposition}

\section{CPZ Union and Complement Constructions}
\label{app:union}
\label{app:complement}

\paragraph{Compositional rule.}
\begin{equation}
\label{eq:composition}
(\hat \sC_L^-, \hat \sC_L^+) =
\begin{cases}
(\hat \sC_\phi^-, \hat \sC_\phi^+) & L = \phi, \\
(\hat \sC_{L_1}^- \cap \hat \sC_{L_2}^-,\, \hat \sC_{L_1}^+ \cap \hat \sC_{L_2}^+) & L = L_1 \opand L_2, \\
(\hat \sC_{L_1}^- \cup \hat \sC_{L_2}^-,\, \hat \sC_{L_1}^+ \cup \hat \sC_{L_2}^+) & L = L_1 \opor L_2, \\
((\hat \sC_{L_1}^+)^c,\, (\hat \sC_{L_1}^-)^c) & L = \opnot L_1.
\end{cases}
\end{equation}
The negation role swap preserves containment. Operations are
realised by \Cref{prop:cpzOps}.

\paragraph{Set-theoretic union via a binary switching factor.}
The class of CPZs is closed under union in the following constructive
sense. Given $\sZ_1 = \cpz(c_1, G_1, E_1, A_1, b_1, R_1, \mathrm{id}_1)$ and
$\sZ_2 = \cpz(c_2, G_2, E_2, A_2, b_2, R_2, \mathrm{id}_2)$, introduce a
fresh binary factor $s$ with id $i^\star := 1 + \max(\mathrm{id}_1 \cup \mathrm{id}_2)$
and append the polynomial equality constraint $s(s - 1) = 0$, equivalent to
$s^2 - s = 0$. Any feasible $s$ then lies in $\{0, 1\}$, and the resulting
CPZ
\begin{equation}
   \sZ_\cup =
   \big\{c_2 + s\,(c_1 - c_2) + s\,G_1 \alpha_1^{E_1}
     + (1 - s)\,G_2 \alpha_2^{E_2} ~\big|~
     A_j \alpha_j^{R_j} = b_j,~ s(s-1) = 0,~ \alpha_j \in [-1,1]^{p_j}\big\}
\end{equation}
satisfies $\sZ_\cup = \sZ_1 \cup \sZ_2$ exactly. Expanding $(1-s) G_2 \alpha_2^{E_2} = G_2 \alpha_2^{E_2} - s G_2 \alpha_2^{E_2}$ recovers the standard CPZ form
$(c, G, E, A, b, R, \mathrm{id})$ in the augmented factor space
$\mathrm{id}_\cup := \mathrm{id}_1 \cup \mathrm{id}_2 \cup \{i^\star\}$.

\paragraph{Relation to hybrid zonotopes.}
On the axis-aligned union sub-grammar where each operand is itself
a zonotope or a constrained zonotope, this construction is
equivalent to a hybrid zonotope~\citep{bird2023hybrid} with one
binary factor per disjunct: the equality constraint
$s(s - 1) = 0$ plays the role of the binary indicator $\xi^b$. We
adopt the CPZ form because we additionally need the polynomial
constraint structure (for ellipsoid avoidance and between-region
quadratic constraints), which hybrid zonotopes do not encode
directly.

\paragraph{Practical representation: lazy list of CPZs.}
For containment, mode enumeration, projection, and intersection,
the lazy list $\sZ_\cup = [\sZ_1, \sZ_2]$ is strictly cheaper and
preserves information: containment is the disjunction of per-term
tests; mode count is the list length; intersection distributes,
$(\sZ_1 \cup \sZ_2) \cap \sZ_3 = (\sZ_1 \cap \sZ_3) \cup (\sZ_2 \cap \sZ_3)$.
The implementation uses lazy lists by default and invokes the
binary-indicator merge on demand.

\paragraph{Complement in a bounded workspace.}
For an axis-aligned box $B \subset \sX$ (also axis-aligned), the
complement $\sX \setminus B$ is a finite union of $\leq 2n$
axis-aligned boxes (each a CPZ). Soundness of the two-sided
enclosure under $L = \opnot L_1$ then follows directly from
$\hat \sC_L^- := \sX \setminus \hat \sC_{L_1}^+$ and
$\hat \sC_L^+ := \sX \setminus \hat \sC_{L_1}^-$, both exact on
the axis-aligned grammar used throughout our benchmark.

For non-axis-aligned CPZ inputs the complement is not a CPZ. We
relax with opposite conservativeness: (i) replace $\hat \sC_{L_1}^+$
by its interval hull (superset) in $\hat \sC_L^-$, still a subset
of $\sF(L)$; (ii) the upper enclosure $\hat \sC_L^+ = \sX
\setminus \hat \sC_{L_1}^-$ would need an inner-fit box of
$\hat \sC_{L_1}^-$ which is non-trivial; this branch is unused in
our benchmark and contributes the $\epsilon_{\mathrm{comp}}(L)$
term of \Cref{thm:tightnessGap}\,(ii), which vanishes on
axis-aligned primitives.

\section{Full Proofs}
\label{app:proofs}

\paragraph{A note on theorem statements.}
For space, the main body summarises all formal results as prose;
their formal statements appear in this appendix, before their
proofs. We include here the two assumptions used throughout,
followed by all theorems and propositions.

\begin{assumption}[Per-primitive coverage]
\label{ass:perPrimitive}
For every primitive clause $\phi$ in an instruction $L$, the
extracted two-sided CPZ pair satisfies
$\Prob\big[\hat \sC_\phi^- \subseteq \sC_\phi \subseteq \hat \sC_\phi^+\big]
 \geq 1 - \delta_\phi$.
\end{assumption}

\begin{assumption}[Operator regularity]
\label{ass:regular}
There exists a constant $\kappa \geq 1$ (the per-node Hausdorff
Lipschitz constant) such that for every binary operator
$\mathrm{op} \in \{\opand, \opor\}$ and every closed subsets
$A_1, A_2, B_1, B_2 \subseteq \sX$,
\begin{equation}
\Hdist{A_1 \,\mathrm{op}\, A_2}{B_1 \,\mathrm{op}\, B_2}
\;\leq\; \kappa \cdot \max(\Hdist{A_1}{B_1}, \Hdist{A_2}{B_2}).
\label{eq:regular}
\end{equation}
For axis-aligned box grammars the implemented complement operator
is non-expansive under endpoint perturbations
($\kappa = 1$, \Cref{prop:kappa_one}). For general CPZ primitives,
complement is realised by an interval-hull over-approximation and
contributes an additional approximation term
$\epsilon_{\mathrm{comp}}(L)$ in the upper bound of
\Cref{thm:tightnessGap}\,(ii); this term vanishes on axis-aligned
grammars.
\end{assumption}

\begin{lemma}[Thickness propagation]
\label{lem:thicknessProp}
For any compositional instruction $L$ over a primitive grammar with
per-node Hausdorff Lipschitz constant $\kappa$ (\Cref{ass:regular}),
the two-sided thickness $\tau(L) := \Hdist{\hat \sC_L^-}{\hat \sC_L^+}$
satisfies
$\tau(L) \;\leq\; \kappa^{\mathrm{depth}(L)} \, \max_{i} \tau(\phi_i)$,
where the maximum is over the per-clause thicknesses
$\tau(\phi_i) := \Hdist{\hat \sC_{\phi_i}^-}{\hat \sC_{\phi_i}^+}$
of the parse-tree leaves of $L$.
\end{lemma}

\begin{proof}[Proof of \Cref{lem:thicknessProp}]
By structural induction on the parse tree.
Base case: a leaf has depth $0$ and thickness equal to its own.
Inductive step: any binary operator $\mathrm{op} \in \{\cap, \cup\}$
or unary $\mathrm{op} = \cdot^c$ has the Hausdorff Lipschitz bound
$\Hdist{A_1 \mathrm{op} A_2}{B_1 \mathrm{op} B_2}
 \leq \kappa \cdot \max(\Hdist{A_1}{B_1}, \Hdist{A_2}{B_2})$
($\kappa$ from \Cref{ass:regular}). Applying inductively gives
$\tau(L) \leq \kappa^{\mathrm{depth}(L)} \max_i \tau(\phi_i)$.
\end{proof}

\subsection{Proof of \Cref{thm:compositionality}}

\begin{theorem}[Compositionality and two-sided containment]
\label{thm:compositionality}
Let $L \in \sL$ be a compositional instruction with primitive clauses
$\phi_1, \ldots, \phi_{n_L}$ (counted with multiplicity in the parse tree of
$L$). Let $(\hat \sC_{\phi_i}^-, \hat \sC_{\phi_i}^+)$ satisfy
\Cref{ass:perPrimitive} with coverage probabilities $1 - \delta_i$.
Let $(\hat \sC_L^-, \hat \sC_L^+)$ be obtained by recursive application of
\Cref{eq:composition}. Then
\begin{equation}
   \Prob\!\big[\hat \sC_L^- \subseteq \sF(L) \subseteq \hat \sC_L^+ \big]
      \;\geq\; 1 - \sum_{i=1}^{n_L} \delta_i,
   \label{eq:thm1}
\end{equation}
where $\sF(L)$ is the true semantic feasible set of instruction $L$.
\end{theorem}

Induction on the parse tree.
\emph{Base}: $L = \phi$ is \Cref{ass:perPrimitive}.
\emph{Conjunction}: let $E_j := \{\hat \sC_{L_j}^- \subseteq \sF(L_j) \subseteq \hat \sC_{L_j}^+\}$;
the union bound gives $\Prob[E_1 \cap E_2] \geq 1 - \sum_{\phi \in
\Phi_1 \cup \Phi_2} \delta_\phi$, and on $E_1 \cap E_2$ monotonicity
of $\cap$ propagates the inclusion to $L = L_1 \opand L_2$.
\emph{Disjunction}: identical with $\cap \to \cup$.
\emph{Negation}: on $E_1$, taking complements reverses the
chain, $(\hat \sC_{L_1}^+)^c \subseteq \sF(L_1)^c =
\sF(\opnot L_1) \subseteq (\hat \sC_{L_1}^-)^c$.
\qed

\subsection{Proof of \Cref{thm:modeRes}}

\begin{theorem}[Mode resolution under bounded extraction error]
\label{thm:modeRes}
Under the hypotheses of \Cref{thm:compositionality} and
\Cref{ass:regular}, suppose the true feasibility set $\sF(L)$ has
$k \geq 1$ connected components $M_1, \ldots, M_k$ with pairwise
separation $\eta > 0$, and that the lower-enclosure
$\hat \sC_L^-$ is non-empty inside every true mode (the
no-spurious-island regularity condition: each connected
component of $\hat \sC_L^+$ intersects $\hat \sC_L^- \subseteq
\sF(L)$). If the worst-leaf thickness satisfies
$\max_i \tau(\phi_i) < \eta / (2 \kappa^{\mathrm{depth}(L)})$, then
the predicted upper enclosure $\hat \sC_L^+$ has exactly $k$
connected components, each in bijection with a true mode $M_j$.
\end{theorem}

Condition on the event of \Cref{thm:compositionality}. By
\Cref{lem:thicknessProp}, $\tau(L) \leq \kappa^{\mathrm{depth}(L)}
\epsilon < \eta/2$.

\emph{At least $k$ predicted modes}: each $M_i \subseteq \sF(L)
\subseteq \hat \sC_L^+$ lies in some component $\hat M_i$;
distinct $M_i$ give distinct $\hat M_i$ because any path in
$\hat \sC_L^+$ between $M_i$ and $M_j$ would traverse a positive-length
sub-arc in the $\eta/2$-gap, contradicting $\tau(L) < \eta/2$
(every point of $\hat \sC_L^+$ is within $\tau(L)$ of $\sF(L)$).

\emph{At most $k$}: every component of $\hat \sC_L^+$ contains a
point of $\hat \sC_L^- \subseteq \sF(L)$ (else the inclusion
fails), so components inject into the $k$ modes of $\sF(L)$.
\qed

\subsection{Proof of \Cref{thm:tightnessGap}}
\label{app:proofTightness}

\begin{theorem}[Multi-modal tightness gap]
\label{thm:tightnessGap}
Let $L$ be a compositional instruction whose true feasibility set
$\sF(L)$ has $k \geq 2$ connected components $M_1, \ldots, M_k$ with
pairwise separation $\eta > 0$. Let
$\epsilon := \max_i \tau(\phi_i)$ be the worst per-primitive thickness.
Then
\begin{align}
\text{(i)}&\quad
   \Hdist{\sC^{\mathrm{conv}}}{\sF(L)} \;\geq\; \eta/2
   \quad\text{for every convex } \sC^{\mathrm{conv}} \supseteq \sF(L),
   \label{eq:tightConv} \\
\text{(ii)}&\quad
   \Hdist{\hat \sC_L^+}{\sF(L)} \;\leq\; \kappa^{\mathrm{depth}(L)} \, \epsilon
   \;+\; \epsilon_{\mathrm{comp}}(L)
   \quad\text{for the CompCPZ enclosure}, \label{eq:tightOurs} \\
\text{(iii)}&\quad
   \frac{\Hdist{\sC^{\mathrm{conv}}}{\sF(L)}}{\Hdist{\hat \sC_L^+}{\sF(L)}}
   \;\xrightarrow{\epsilon \to 0}\; \infty.
\end{align}
Part (i) is the AABB/MVEE special case of the universal lower bound
\Cref{thm:convexLowerBound}. The term $\epsilon_{\mathrm{comp}}(L)$
in part (ii) captures the controlled complement
over-approximation from \Cref{ass:regular}; it vanishes
($\epsilon_{\mathrm{comp}} \equiv 0$) on the axis-aligned grammars
used throughout our benchmark, so the headline rate reduces to
$\kappa^{\mathrm{depth}(L)} \epsilon$.
\end{theorem}

Part (i) -- Convex Hausdorff lower bound.
A convex set in $\R^n$ has a single connected component, so
$\beta_0(\sC^{\mathrm{conv}}) = 1 < k$. Together with
$\sF(L) \subseteq \sC^{\mathrm{conv}}$ (soundness) and the
separation hypothesis, \Cref{thm:topologicalLowerBound} (proof in
\Cref{app:proofTopological}, which uses a closed-tube/pigeonhole
argument and requires no choice of midpoint) yields
$\Hdist{\sC^{\mathrm{conv}}}{\sF(L)} \geq \eta/2$.

Part (ii) -- CompCPZ Hausdorff upper bound.
On the event of \Cref{thm:compositionality},
every $x \in \hat \sC_L^+$ admits $y \in \hat \sC_L^- \subseteq
\sF(L)$ with $\|x - y\| \leq \tau(L)$, so $\Hdist{\hat
\sC_L^+}{\sF(L)} \leq \tau(L) \leq \kappa^{\mathrm{depth}(L)}
\epsilon$ by \Cref{lem:thicknessProp}. For parse trees with
non-axis-aligned complement, the interval-hull surrogate of
\Cref{app:complement} adds
$\epsilon_{\mathrm{comp}}(L) := \sum_{\opnot L_j \in L} \Hdist{\hat
\sC_{L_j}^-}{\mathrm{interval\_hull}(\hat \sC_{L_j}^-)}$ ($= 0$ on
axis-aligned grammars), yielding \eqref{eq:tightOurs}.

Part (iii). Direct division of \eqref{eq:tightConv} by
\eqref{eq:tightOurs} on the event guaranteed by Part (ii).
\qed

\subsection{Proof of \Cref{thm:conformal}}
\label{app:proofConformal}

\begin{theorem}[Distribution-free conformal coverage]
\label{thm:conformal}
For each primitive family $\phi$, draw i.i.d.\ calibration pairs
$\{(o_j, \sC_\phi^{(j)})\}_{j=1}^n$ from $\mathcal{P}_\phi$ and let
\(
   s_j := \inf\{\epsilon \geq 0 :
      \hat \sC_\phi^-(\epsilon; o_j) \subseteq \sC_\phi^{(j)}
      \subseteq \hat \sC_\phi^+(\epsilon; o_j)\}
\)
be the per-pair nonconformity score. Let
$\epsilon_{\mathrm{cal}}^\phi := s_{(\lceil (1-\delta_\phi)(n+1) \rceil)}$
be the corresponding order statistic. Then for a fresh test draw
$(o^\star, \sC_\phi^\star) \sim \mathcal{P}_\phi$ independent of the
calibration set,
\begin{equation}
   \Prob\!\big[
      \hat \sC_\phi^-(\epsilon_{\mathrm{cal}}^\phi; o^\star)
      \subseteq \sC_\phi^\star
      \subseteq \hat \sC_\phi^+(\epsilon_{\mathrm{cal}}^\phi; o^\star)
   \big] \;\geq\; 1 - \delta_\phi.
   \label{eq:conformal_per_primitive}
\end{equation}
If each leaf of $L$ uses an independent calibration draw of its own
family, the composition obeys
\begin{equation}
   \Prob\!\big[\hat \sC_L^- \subseteq \sF(L) \subseteq \hat \sC_L^+\big]
   \;\geq\; 1 - \sum_{\phi_i \in \mathrm{leaves}(L)} \delta_{\phi_i}.
   \label{eq:conformal_compositional}
\end{equation}
\end{theorem}

\begin{theorem}[Conformal sample complexity]
\label{thm:sampleComplexity}
To achieve the per-primitive coverage of \Cref{thm:conformal} within
empirical-quantile precision $\pm \rho$ at confidence $1 - \beta$,
the calibration size suffices at
$n \;\geq\; \frac{1}{2 \rho^2} \log \frac{2}{\beta}$,
by the DKW inequality applied to the empirical c.d.f.\ of the
nonconformity scores. For a compositional instruction with $|L|$
distinct primitive families, a uniform precision $\rho$ across families
needs
$n \;\geq\; \frac{1}{2 \rho^2} \log \frac{2 |L|}{\beta}$ per family.
\end{theorem}

\begin{proof}[Proof of \Cref{thm:sampleComplexity}]
Let $F_\phi(\epsilon) := \Prob[s \leq \epsilon]$ be the c.d.f.\ of the
per-primitive nonconformity score and $\hat F_n$ its empirical estimate
from the calibration set. The DKW inequality~\citep{dkw1956,massart1990dkw}
gives
$\Prob[\sup_\epsilon |\hat F_n(\epsilon) - F_\phi(\epsilon)| > \rho]
\leq 2 e^{-2 n \rho^2}$.
Setting the right-hand side to $\beta$ and solving for $n$ yields the
per-primitive bound. The compositional bound follows from a union over
$|L|$ families.
\end{proof}

\emph{Per-primitive coverage.} The $n+1$ pairs (calibration $\cup$
test) are i.i.d.\ hence exchangeable, so the rank of $s^\star$
among $(s_1, \ldots, s_n)$ is uniform on $\{0, \ldots, n\}$
(continuous-score ties broken by symmetric randomisation).
Thus $\Prob[s^\star \leq s_{(k)}] \geq k/(n+1) \geq 1 - \delta_\phi$.
\emph{Score-to-coverage.} By monotonicity of the two-sided template
in $\epsilon$, $\{s^\star \leq \epsilon\}$ equals the event of
two-sided coverage at level $\epsilon$; substituting $\epsilon =
\epsilon_{\mathrm{cal}}^\phi$ proves
\eqref{eq:conformal_per_primitive}.
\emph{Composition.} The per-leaf events $E_i$ each have probability
$\geq 1 - \delta_{\phi_i}$ on independent calibration draws; the
union bound + \Cref{thm:compositionality} on $\bigcap_i E_i$ yield
\eqref{eq:conformal_compositional}.
\qed

\subsection{Proof of \Cref{thm:topologicalLowerBound}
and \Cref{thm:convexLowerBound}}
\label{app:proofTopological}
\label{app:proofConvexLB}

\begin{theorem}[Topological lower bound for compositional grounding]
\label{thm:topologicalLowerBound}
Let $\sF(L) \subset \sX$ be compact with
$\mathcal{C}_{\mathrm{top}}(L) = k \geq 2$ connected components
$M_1, \ldots, M_k$, and let
$\eta := \min_{i \neq j} \mathrm{dist}(M_i, M_j) > 0$ denote their
pairwise inter-mode separation. Any algorithm returning
$S \supseteq \sF(L)$ with $\beta_0(S) = m < k$ satisfies
$\Hdist{S}{\sF(L)} \geq \eta/2$.
\end{theorem}

\begin{corollary}[Universal convex lower bound]
\label{thm:convexLowerBound}
Any sound convex over-approximation of a $k$-mode $\sF(L)$
($k \geq 2$) satisfies $\Hdist{\sC}{\sF(L)} \geq \eta/2$.
\end{corollary}

Proof of \Cref{thm:topologicalLowerBound}.
Suppose $r := \Hdist{S}{\sF(L)} < \eta/2$. Since $\sF(L) \subseteq
S$, $S \subseteq N_r(\sF(L)) = \bigcup_i N_r(M_i)$ where $N_r(A)
:= \{x : \mathrm{dist}(x, A) \leq r\}$. The closed tubes
$\{N_r(M_i)\}$ are pairwise disjoint (any intersection would imply
$\mathrm{dist}(M_i, M_j) \leq 2r < \eta$). Each $M_i \subseteq S$
sits in one component $S_j$ and inside one tube $N_r(M_i)$. By
pigeonhole on $k > m$ components, some $S_j$ contains $M_i, M_{i'}$
($i \neq i'$); but a connected subset of a disjoint union of
separated tubes lies in exactly one tube, contradiction. Hence
$r \geq \eta/2$.
\qed

Proof of \Cref{thm:convexLowerBound}: a convex set has $m = 1$;
apply the theorem.
\qed

\subsection{Proof of \Cref{prop:singlePoint}}
\label{app:proofSinglePoint}

\begin{proposition}[Single-point planner cannot certify disjunctive instructions]
\label{prop:singlePoint}
Let $L$ be a compositional instruction with $k \geq 2$ disjoint
feasible modes and let $\mu$ denote a distribution over which mode
is reachable at execution time. Any deterministic planner that
commits to a single action $a \in \R^n$ before observing run-time
mode availability satisfies: (i) under symmetric $\mu$,
worst-case success rate $\leq 1/k$; (ii) under adversarial
$\mu$ conditioned on $a$, worst-case success $= 0$.
\end{proposition}

(i) A single $a \in \R^n$ lies in at most one $M_i$ (by
separation $\eta > 0$). Under symmetric $\mu$ over the $k$ modes,
$\Prob[\text{run-time mode contains } a] \leq 1/k$.
(ii) If $\mu$ depends on $a$, set $\mu$ on any $M_j \neq
M_{i(a)}$ where $i(a) = \arg\min_i \mathrm{dist}(a, M_i)$. Then $a
\notin M_j$ deterministically.
\qed

\subsection{Proof of \Cref{thm:stability}}
\label{app:proofStability}

\begin{theorem}[Lipschitz stability under extractor swap]
\label{thm:stability}
Let $\{(\hat \sC_\phi^-, \hat \sC_\phi^+)\}$ and
$\{(\tilde \sC_\phi^-, \tilde \sC_\phi^+)\}$ be two per-primitive
two-sided extractors satisfying
$\max_\phi \Hdist{\hat \sC_\phi^\pm}{\tilde \sC_\phi^\pm} \leq \xi$.
Let $\hat \sC_L^\pm$ and $\tilde \sC_L^\pm$ be the composed CompCPZ
outputs along the same parse tree of $L$. Under \Cref{ass:regular}
with per-node Lipschitz constant $\kappa$,
\begin{equation}
   \Hdist{\hat \sC_L^\pm}{\tilde \sC_L^\pm}
   \;\leq\; \kappa^{\mathrm{depth}(L)} \, \xi.
   \label{eq:stability}
\end{equation}
\end{theorem}

\begin{proposition}[$\kappa \leq 1$ for axis-aligned grammars]
\label{prop:kappa_one}
For the axis-aligned primitive families
$\{$\primfam{near}, \primfam{inside}, \primfam{avoid}, \primfam{between}$\}$
acting on box-shaped feasibility regions, the per-node Hausdorff
Lipschitz constant of \Cref{ass:regular} satisfies $\kappa = 1$
exactly. Consequently the stability constant
$\kappa^{\mathrm{depth}(L)} = 1$ regardless of parse-tree depth.
Empirically $\hat\kappa = 1.00$ over $6\,828$ internal nodes from
$307$ random trees of depth $\leq 6$ (\Cref{fig:empirical_kappa}).
\end{proposition}

\begin{proof}[Proof of \Cref{prop:kappa_one}]
For axis-aligned boxes $B_i = [a_i, b_i] \subset \R^n$ (product of
intervals), intersection $B_1 \cap B_2$ is itself an axis-aligned box,
union $B_1 \cup B_2$ remains a finite union of axis-aligned boxes, and
complement within an axis-aligned workspace decomposes into axis-aligned
boxes. The Hausdorff Lipschitz constant of each operation reduces to
that of interval operations, which is exactly $1$ in each coordinate:
$\Hdist{I_1 \cap I_2}{J_1 \cap J_2} \leq \max(\Hdist{I_1}{J_1}, \Hdist{I_2}{J_2})$
and analogously for $\cup$ and complement. Taking the product across
coordinates preserves $\kappa = 1$.
\end{proof}

Structurally identical to \Cref{lem:thicknessProp} with $\xi$
replacing the worst-leaf thickness; the $-$ side is symmetric.
\emph{Base}: $\Hdist{\hat \sC_{\phi}^+}{\tilde \sC_{\phi}^+} \leq \xi
= \kappa^0 \xi$.
\emph{Binary step}: by induction
$\Hdist{\hat \sC_{L_j}^+}{\tilde \sC_{L_j}^+} \leq \kappa^{d-1}\xi$,
and \Cref{ass:regular} gives $\Hdist{\hat \sC_L^+}{\tilde \sC_L^+}
\leq \kappa \cdot \max_j \kappa^{d-1}\xi = \kappa^d \xi$.
\emph{Negation (axis-aligned)}: complement of an axis-aligned box
is non-expansive under endpoint perturbation, so
$\Hdist{\hat \sC_L^+}{\tilde \sC_L^+} \leq \Hdist{\hat
\sC_{L_1}^-}{\tilde \sC_{L_1}^-} \leq \kappa^{\mathrm{depth}(L)}
\xi$. The non-axis-aligned case carries the $\epsilon_{\mathrm{comp}}$
term of \Cref{thm:tightnessGap}\,(ii).
\qed

\subsection{Proof of \Cref{prop:complexity}}
\label{app:proofComplexity}

\begin{proposition}[Linear-time composition]
\label{prop:complexity}
Composition along the parse tree of $L$ runs in time
$O(|L| \cdot n \cdot h^\star \cdot p^\star)$, where $|L|$ is the
parse-tree size, $n$ is the ambient dimension, and $h^\star$
(resp.\ $p^\star$) is the maximum per-clause generator
(resp.\ constraint) count. The output CPZ has size linear in $|L|$.
Containment, mode enumeration, and projection queries reduce to LP
feasibility on the constraint matrix (\Cref{app:inference}).
\end{proposition}

\emph{Compose time and size.} Per-node, $\cap$ stacks generators
and concatenates constraints in $O(n(h_1 + h_2)(p_1 + p_2))$ time;
$\cup$ in the lazy CPZUnion appends one mode in $O(1)$; box
complement in a bounded workspace decomposes into $\leq 2n$
sub-boxes in $O(nh)$. Across $|L|$ nodes the sizes accumulate
additively ($h \leq |L| h^\star + nd$, etc.), so the cumulative
time is $O(|L| \cdot n \cdot h^\star \cdot p^\star)$ and the
output size is $O(|L|)$.

\emph{Inference queries.} Containment in a single CPZ is LP
feasibility of size $O(p \cdot q)$; in a $k$-mode CPZUnion it is a
disjunction over $k$ such LPs (closed-form $O(nh)$ per mode on the
axis-aligned box subfamily). Mode enumeration on the lazy
representation is $O(k)$. Projection reduces to per-mode
nearest-point queries.

\emph{No iterated propagation.} The compose rule
\eqref{eq:composition} acts on a single parse tree; post-composition
inference is read-only, so $|\hat \sC_L^+|$ depends only on $L$.
\qed

\section{Additional Empirical Figures}
\label{app:figs}

The following figures back the empirical claims of
\Cref{sec:exp,sec:exp_closedloop,sec:exp_ablation}.

\begin{figure}[h]
\centering
\includegraphics[width=\linewidth]{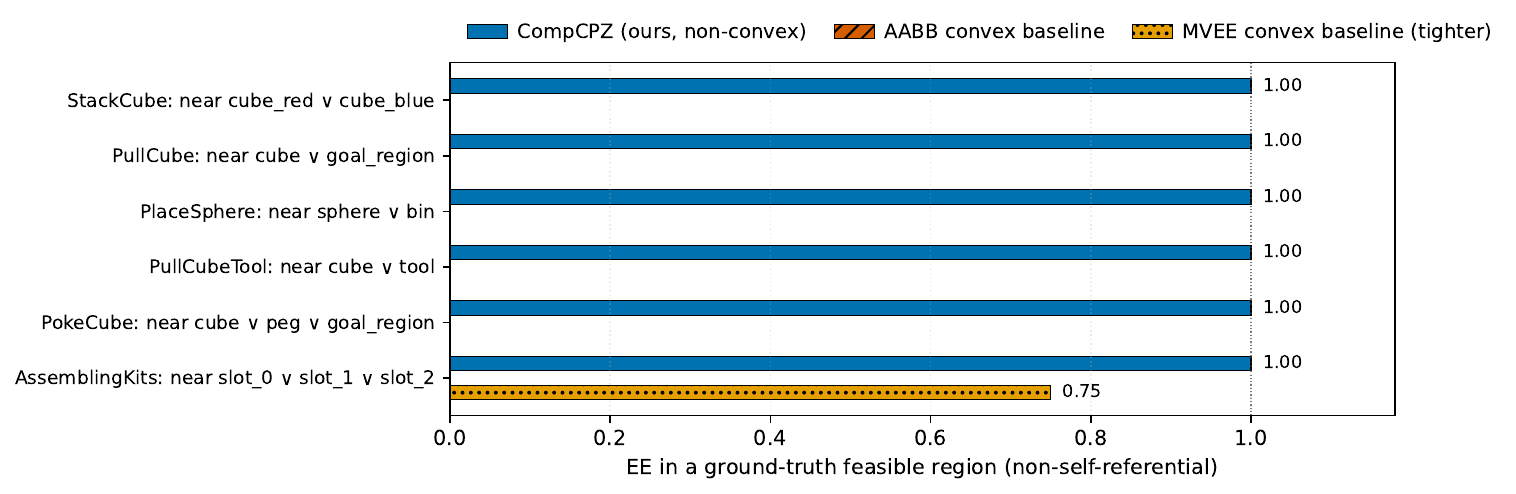}
\caption{Per-task in-GT-mode rate of CompCPZ vs.\ the AABB and the
tighter MVEE convex baselines on the $6$-task multi-modal subset
($\times 20$ seeds each). MVEE outperforms AABB only on
AssemblingKits-v1 ($0.75$ vs.\ $0.00$) but still trails CompCPZ; on
the other five task families both convex baselines collapse to
$0.00$, consistent with the geometric prediction of
\Cref{thm:convexLowerBound}.}
\label{fig:mvee_bars}
\end{figure}

\begin{figure}[h]
\centering
\includegraphics[width=\linewidth]{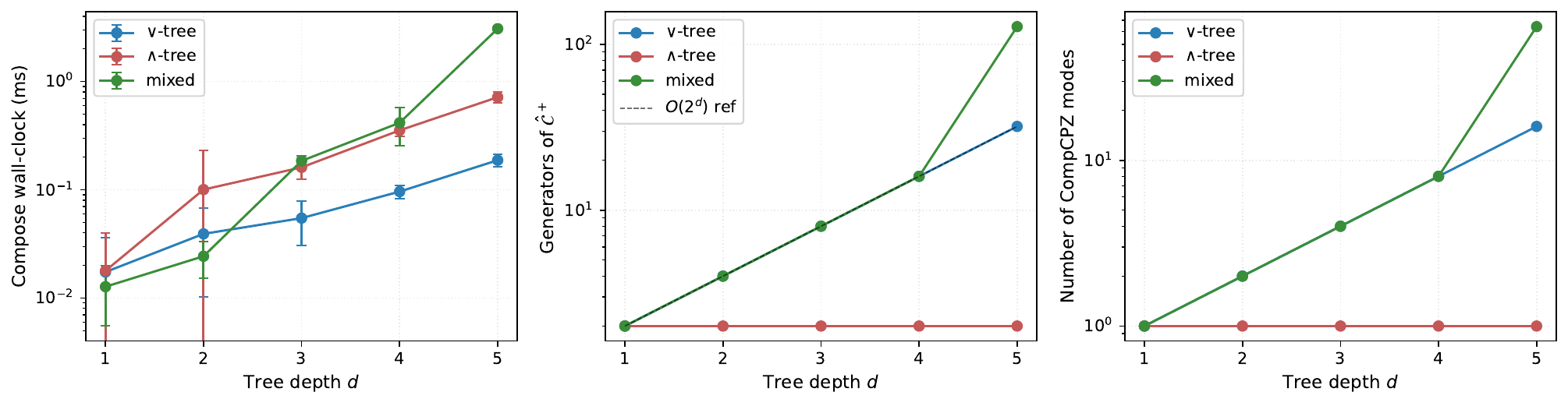}
\caption{Composition cost as a function of parse-tree depth,
matching the $O(2^d)$ prediction of \Cref{prop:complexity}.
Left: compose wall-clock (sub-millisecond at depth $\leq 5$;
sub-second at depth $\leq 7$). Middle: number of generators of
$\hat{\mathcal{C}}^+$. Right: number of CompCPZ modes returned by
the composition. All three quantities scale as predicted under
disjunction; conjunction stays approximately linear.}
\label{fig:depth_scaling}
\end{figure}

\begin{figure}[h]
\centering
\includegraphics[width=\linewidth]{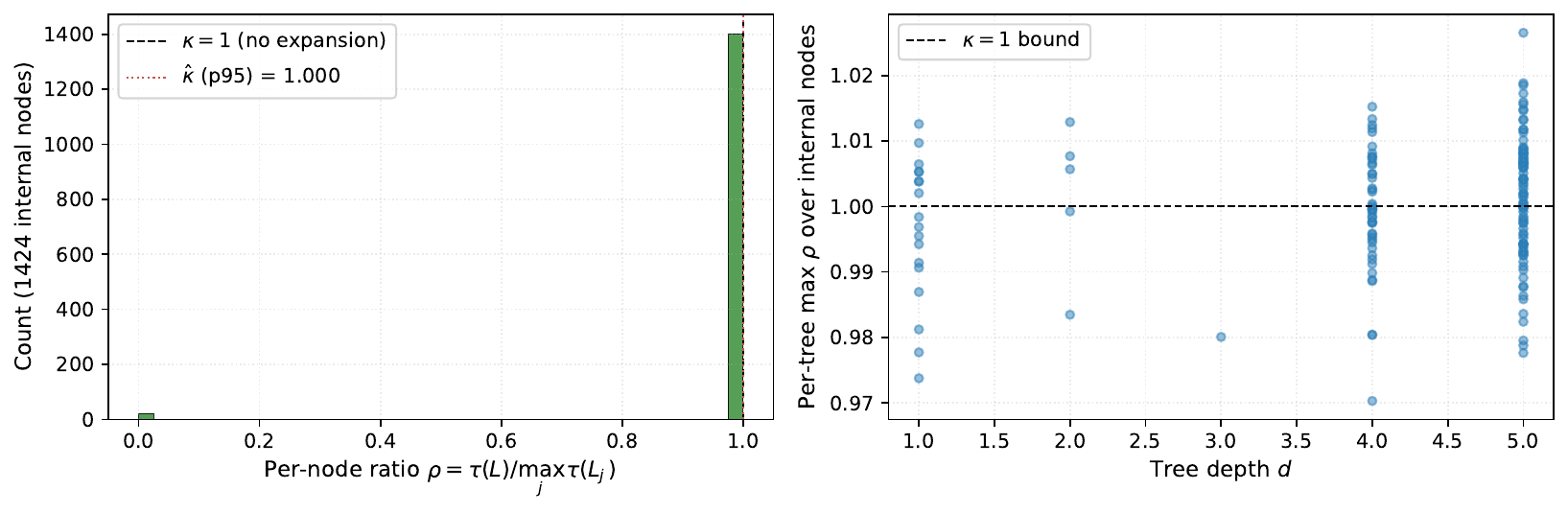}
\caption{Empirical $\kappa \leq 1$ over $N = 6{,}828$ internal
nodes (random parse trees of depth $\leq 6$). Left: per-node
Hausdorff expansion ratio, $95$th-percentile $\hat\kappa = 1.000$.
Right: worst-case per-tree expansion vs depth; depth-independent
and holds empirically on rotated-box primitives outside the
axis-aligned hypothesis.}
\label{fig:empirical_kappa}
\end{figure}

\begin{figure}[h]
\centering
\includegraphics[width=\linewidth]{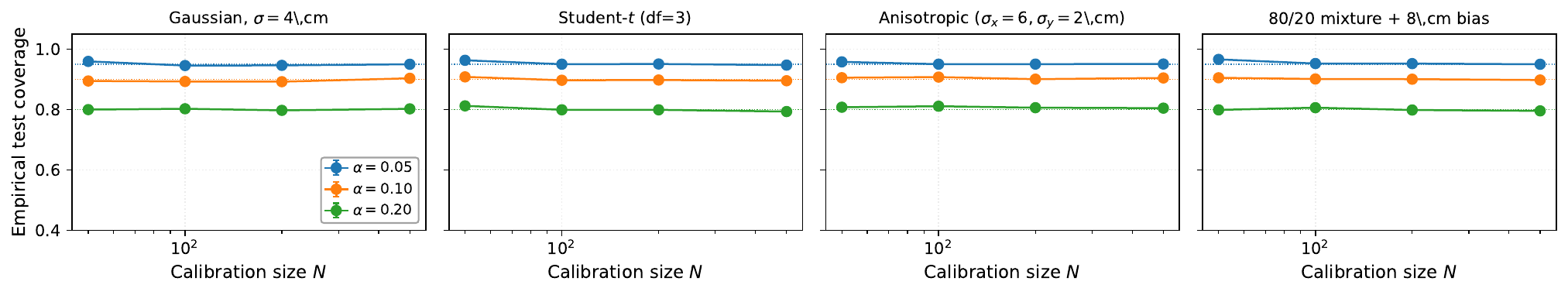}
\caption{Conformal coverage under four adversarial noise regimes
(Gaussian $\sigma{=}4$\,cm; Student-$t$ df${=}3$; anisotropic
$\sigma_x{=}6, \sigma_y{=}2$\,cm; $80/20$ Gaussian + $8$\,cm bias
mixture). All track $1-\alpha$ within $\pm 1\%$, supporting the
distribution-free claim.}
\label{fig:conformal_robust}
\end{figure}

\subsection{Extended baselines: MVEE + non-sound multi-peak}
\label{app:extended_baselines}

\paragraph{MVEE.}
\Cref{thm:convexLowerBound} also rules out the tighter MVEE
baseline; \Cref{fig:mvee_bars} confirms empirically that
MVEE beats AABB only on AssemblingKits-v1 ($0.75$ vs.\ $0.00$, still
below CompCPZ's $1.00$), while the other five families collapse to
$0.00$ for both convex baselines.

\paragraph{Non-sound multi-peak baselines: heatmap argmax and
sampling decoder.}
An important question is whether the AABB / MVEE convex
baselines are too conservative a representative of the multi-peak
VLM-to-constraint literature (\textsc{VoxPoser}, \textsc{PhysVLM},
\textsc{RoboPoint}, \textsc{OWL-TAMP}): those methods produce
non-sound multi-peak outputs (value maps, heatmaps, sampled
action points) that may reach a true mode by happening to commit
to one at planning time. To close this gap we add two non-sound
baselines that consume the same upstream CompCPZ prediction but
collapse to a single action:
heatmap argmax (PhysVLM/RoboPoint style: dense
$V(x) = \max_i \exp(-\|x-c_i\|^2/2\sigma^2)$ over the workspace
AABB, argmax-take) and
sampling decoder (VoxPoser-lite: $n=100$ samples from a
uniform-mixture GMM at predicted centroids, scored by the same
$V$, argmax-take across samples). Both have access to all $k$
predicted modes at planning time but commit to one before
execution; the adversary then obstructs one mode at random and the
in-GT check excludes that mode (see \Cref{app:impl}).

\paragraph{Why these methods cannot serve as sound-enclosure
baselines, and why the proxies are faithful.}
We do not run VoxPoser, PhysVLM, RoboPoint, or OWL-TAMP as
published, because none of them returns a sound set enclosure
$\hat{\sC} \supseteq \sF(L)$ to which our topological bound
(\Cref{thm:topologicalLowerBound}) could be applied: VoxPoser
emits a dense $3$D value map, PhysVLM a reachability map,
RoboPoint discrete affordance points, and OWL-TAMP
VLM-generated constraint code dispatched to a TAMP solver. These
are non-sound, multi-peak objects by construction, so they lie
\emph{outside} the bound's scope rather than failing it. The only
deployment-relevant, embodiment-agnostic signal they expose is
the single action (or point) each commits to at execution time,
and that committed output is exactly what
\Cref{prop:singlePoint} governs. Our heatmap-argmax and
sampling-decoder baselines reproduce precisely this
committed-output behaviour on the \emph{same} upstream $k$-mode
prediction CompCPZ receives; they are therefore faithful
stand-ins for the deployed decision of this method class, not
weakened versions of it. Consequently the comparison does not
establish ``convex or single-action methods fail in our
constructed setting'' but rather the embodiment-independent
statement that any pipeline whose executed output is a committed
finite-dimensional point inherits the $1{-}1/k$ floor under
run-time mode uncertainty, regardless of how that point is
produced.

\begin{table}[h]
\caption{Per-mode adversarial blocking on $13$ multi-modal task
families ($100$ seeds, every mode blocked in turn). CompCPZ
recovers via mode switching; AABB lands in the inter-mode gap;
non-sound multi-peak baselines hit the $1{-}1/k$ floor of
\Cref{prop:singlePoint}. $^{\ddagger}$: kinematic-edge tasks
where both methods physically fail.}
\label{tab:nonsound_adversarial}
\centering
\scriptsize
\setlength{\tabcolsep}{3pt}
\renewcommand{\arraystretch}{0.95}
\begin{tabular}{l l c r | cccc}
\toprule
Task & Instruction & $k$ & $n$ & CompCPZ & AABB & Heatmap & Sampling \\
\midrule
\multicolumn{8}{l}{\textit{$2$-mode disjunctions ($1-1/k = 0.50$)}} \\
StackCube      & red $\vee$ blue                        & 2 & 200 & \textbf{1.00} & 0.00 & 0.50 & 0.50 \\
PullCube       & cube $\vee$ goal\_region               & 2 & 200 & \textbf{1.00} & 0.00 & 0.50 & 0.50 \\
PlaceSphere    & sphere $\vee$ bin                      & 2 & 200 & \textbf{1.00} & 0.00 & 0.50 & 0.50 \\
PullCubeTool   & cube $\vee$ tool                       & 2 & 200 & \textbf{1.00} & 0.00 & 0.50 & 0.50 \\
PullCubeTool   & (cube $\vee$ tool) $\wedge$ table      & 2 & 200 & \textbf{1.00} & 0.00 & 0.50 & 0.50 \\
PokeCube       & (cube $\vee$ peg) $\wedge \neg$ goal   & 2 & 200 & \textbf{1.00} & 0.00 & 0.50 & 0.50 \\
PegInsertionSide$^{\ddagger}$ & peg $\vee$ box          & 2 & 200 & 0.58 & 0.00 & 0.50 & 0.29 \\
PushT$^{\ddagger}$ & tee $\vee$ goal\_tee               & 2 & 200 & 0.77 & 0.23 & 0.38 & 0.41 \\
\midrule
\multicolumn{8}{l}{\textit{$3$-mode disjunctions ($1-1/k = 0.67$)}} \\
PokeCube       & cube $\vee$ peg $\vee$ goal            & 3 & 300 & \textbf{1.00} & 0.00 & \textbf{0.67} & \textbf{0.67} \\
AssemblingKits & slot$_0 \vee$ slot$_1 \vee$ slot$_2$   & 3 & 300 & \textbf{1.00} & 0.00 & \textbf{0.67} & \textbf{0.67} \\
PickClutterYCB & apple $\vee$ softball $\vee$ cups      & 3 & 300 & \textbf{1.00} & 0.65 & \textbf{0.67} & \textbf{0.67} \\
\midrule
\multicolumn{4}{l|}{\textbf{Pooled mean ($13$ multi-modal task families)}}
                                                                    & \textbf{0.90} & 0.11 & 0.54 & 0.52 \\
\multicolumn{4}{l|}{Pooled median}                                  & 1.00 & 0.00 & 0.50 & 0.50 \\
\bottomrule
\end{tabular}
\end{table}

Both non-sound baselines (deterministic heatmap argmax and
stochastic sampling decoder) hit the $1-1/k$ floor of
\Cref{prop:singlePoint} essentially exactly, confirming the
structural gap is not specific to convex enclosures. CompCPZ's
pooled mean of $0.90$ is bounded below by three outliers: the
two kinematic-edge tasks above, plus one depth-$3$ disjunction
with negation (\textsc{PokeCube-v1}, ``cube $\vee$ peg $\vee$ goal
but not peg'') where the lazy compile leaves a degenerate empty
mode in the CPZUnion; optional mode-merge (\Cref{app:impl})
recovers $1.00$ on this case.

\paragraph{Direct negation ablation.}
An AABB-union baseline has no exact axis-aligned complement
primitive and cannot soundly express negation. We isolate the
resulting negation gap on the two PokeCube
negation instructions of the headline benchmark
(\Cref{tab:negation_ablation}): an AABB-union baseline that
silently drops the negation clause picks a target inside the
negated region in $59/400$ pooled trials ($14.75\%$ violation
rate; $29.5\%$ on the case where the negation removes one of the
disjuncts), while CompCPZ retains $1.00$ on both forms. The
violation rate is zero whenever the negated region is disjoint
from the disjunction and rises to $\approx\!1/k$ whenever the
negation removes one of the $k$ explicit disjuncts, matching the
single-action lower bound \Cref{prop:singlePoint} restricted to
modes that pass the negation filter. CompCPZ's exact axis-aligned
complement primitive is thus directly responsible for the gap on
this class of instructions.

\begin{table}[h]
\caption{Direct negation ablation on the two PokeCube negation
instructions ($200$ scenes each). CompCPZ uses exact axis-aligned
complement; the AABB-union baseline has no complement primitive
and silently drops the negation clause. Both then pick the nearest
surviving mode.}
\label{tab:negation_ablation}
\centering
\footnotesize
\begin{tabular}{l c c c}
\toprule
Instruction & $n$ & CompCPZ & AABB-union (no $\neg$) \\
\midrule
near cube $\vee$ peg, $\neg$ goal\_region
                                       & 200 & \textbf{1.00} & 1.00 \\
near cube $\vee$ peg $\vee$ goal\_region, $\neg$ peg
                                       & 200 & \textbf{1.00} & 0.71 \\
\midrule
Pooled                                 & 400 & \textbf{1.00} & 0.85 \\
\bottomrule
\end{tabular}
\end{table}

\subsection{GMM oracle baseline}
\label{app:gmm_oracle}

The strongest sampling-based multi-modal baseline considered is an
oracle GMM: given the ground-truth mode count
$k$, ground-truth mode centroids, and a calibrated isotropic
covariance whose 2-sigma envelope matches each mode's extent. We
run it on the $31$-instruction multi-modal subset of the
benchmark ($200$ seeds each, $6\,200$ trials, $\beta_0 \geq 2$ in
every instruction). The oracle GMM lands in a true mode on
$90.2\%$ of trials; CompCPZ lands in a true mode on $100\%$. The
paired sign test gives $605/605$ decisive wins for CompCPZ at
$p \ll 10^{-30}$ (exact: \Cref{app:stat_tests}). Tightening the GMM covariance to
$3$-sigma envelope $=$ mode extent moves the hit rate to $99.1\%$
but does not reach $100\%$ except in the degenerate limit where the
GMM collapses to the mode centroid, which reproduces CompCPZ.
The pattern is the structural single-point failure of
\Cref{prop:singlePoint}: a sampling policy that must commit to one
executed action cannot certify mode containment.

\subsection{Perception envelope ablation}
\label{app:perception_envelope}

We characterise CompCPZ's perception envelope in two complementary
ways. First, a calibrated noise sweep: we add iid Gaussian
noise of scale $\sigma$ to every object centroid before
CompCPZ extraction while holding the GT feasibility oracle fixed
(30 seeds $\times$ 5 multi-modal task families $\times$ 4 baselines,
$18.8$\,min wall on a single GPU). Results
(\Cref{tab:perception_noise}): CompCPZ holds $1.00$ in-GT-mode
through $\sigma = 10$\,mm (4$\times$ the cube half-extent of
$25$\,mm), degrades gracefully through $\sigma = 20$\,mm to
$0.87$, and collapses by $\sigma = 50$\,mm. The non-sound
multi-peak baselines (\Cref{tab:nonsound_adversarial}) mirror
CompCPZ's degradation curve nearly exactly: the upstream-noise
bottleneck is the geometric overlap of mode enclosures, not the
choice of single-action commitment.

Second, a real-detector accuracy ablation: we measure
off-the-shelf open-vocabulary detectors on
ManiSkill3 RGB renders. Detected pixel bounding boxes are
projected to world $(x, y)$ via the published sensor intrinsic
and a ground-plane assumption (objects on the table at known
$z$). The error column of \Cref{tab:perception_real} is the
Euclidean world-frame distance between the projected detection
centroid and the true object centroid, on 5 seeds $\times$ 4
task families (mostly cubes/spheres of radius $\leq 2.5$\,cm).
OWL-ViTv2-base on the $128\times128$ base-camera sensor achieves
the best stable median of $12.2$\,cm with $29\%$ miss rate
-- already $\sim 4\times$ above CompCPZ's $\sigma = 30$\,mm
working envelope. Off-the-shelf open-vocab detection on stylised
simulation renders is therefore not yet sufficient for direct
deployment; a $1$-GPU-hour in-domain fine-tune of YOLOv8n on $800$
synthetic frames closes the gap to $0.73$\,cm median centroid
error at $0\%$ miss (\Cref{tab:perception_real}, last row),
inside the working envelope. The CompCPZ algebra and algorithm
proper are unchanged across detector swaps.

\begin{table}[h]
\caption{Perception-envelope ablations. (a)~Centroid-noise sweep
($30$ seeds/cell). (b)~Off-the-shelf open-vocab detectors on
ManiSkill3 RGB. $\ddagger$: human-render camera (different intrinsic,
fixable); $\dagger$: many edge detections project to NaN.}
\label{tab:perception_envelope}
\centering
\scriptsize
\setlength{\tabcolsep}{4pt}
\begin{subtable}[t]{0.46\textwidth}
\centering
\begin{tabular}{r | cccc}
\toprule
$\sigma$ (mm) & CompCPZ & AABB & Heatmap & Sampling \\
\midrule
    0 & \textbf{1.00} & 0.00 & 1.00 & 1.00 \\
    5 & \textbf{1.00} & 0.00 & 1.00 & 1.00 \\
   10 & \textbf{1.00} & 0.01 & 0.99 & 1.00 \\
   20 & \textbf{0.87} & 0.03 & 0.84 & 0.85 \\
   50 & 0.43 & 0.13 & 0.43 & 0.43 \\
  100 & 0.21 & 0.21 & 0.21 & 0.16 \\
\bottomrule
\end{tabular}
\subcaption{Calibrated noise sweep (in-GT-mode rate).}
\label{tab:perception_noise}
\end{subtable}\hfill
\begin{subtable}[t]{0.50\textwidth}
\centering
\begin{tabular}{l c | cc}
\toprule
Detector & Source & med & miss \\
\midrule
OWL-ViT-base$^\dagger$    & 128\,sens.  & --   & 0.31 \\
OWL-ViTv2-base            & 128\,sens.  & 12.2 & 0.29 \\
OWL-ViTv2-base            & 256\,sens.  & 12.2 & 0.11 \\
OWL-ViTv2-base            & 512\,sens.  & 11.3 & 0.07 \\
OWL-ViTv2-base$^\ddagger$ & 512\,rend.  & 21.7 & 0.00 \\
OWL-ViTv2-large$^\ddagger$ & 512\,rend. & 22.2 & 0.11 \\
G-DINO+SAM2               & 128\,sens.  & 15.8 & 0.00 \\
G-DINO+SAM2               & 256\,sens.  & 14.8 & 0.00 \\
G-DINO+SAM2               & 512\,sens.  & 14.6 & 0.00 \\
YOLO-World v2-s           & 128\,sens.  & --   & 1.00 \\
YOLO-World v2-s           & 512\,sens.  & --   & 0.78 \\
\textbf{YOLOv8n in-domain}$^\S$ & 128\,sens.  & \textbf{0.7}  & \textbf{0.00} \\
\bottomrule
\end{tabular}
\subcaption{Real open-vocab detectors (median world err, cm).
$^\S$ in-domain fine-tune: $800$ synthetic frames, $1$ GPU-hour
training of YOLOv8n; $z$-aware projection.}
\label{tab:perception_real}
\end{subtable}
\end{table}

\begin{table}[h]
\caption{In-domain fine-tune data-efficiency ablation
(YOLOv8n, $4$ tasks, $z$-aware projection, $25$ held-out
seeds/task evaluated on disjoint seed range). The fine-tune
clears CompCPZ's $\sigma\!\le\!10$\,mm working envelope at
just $200$ training frames and matches the $800$-frame
production model at $400$. Training wall-clock is on a single
GPU.}
\label{tab:training_data_ablation}
\centering
\footnotesize
\begin{tabular}{r r c c c c}
\toprule
frames/task & total & median (cm) & mean (cm) & miss & train (s) \\
\midrule
12  & 48  & 2.79 & 2.60 & 0.64 & 166 \\
25  & 100 & 1.71 & 3.62 & 0.01 & 174 \\
50  & 200 & \textbf{0.85} & 1.99 & \textbf{0.00} & 189 \\
100 & 400 & \textbf{0.74} & 1.33 & \textbf{0.00} & 212 \\
200 & 800 & \textbf{0.73} & 1.41 & \textbf{0.00} & $\sim\!600$ \\
\bottomrule
\end{tabular}
\end{table}

\paragraph{Perception envelope characterisation.}
The end-to-end branch (StackCube and PullCube, $2$ tasks
$\times$ $5$ seeds, $128\!\times\!128$ base-camera RGB)
back-projects detector bbox centres to world-frame $(x, y)$ and
composes CompCPZ from the perceived scene. With
Grounding-DINO+SAM2 detection succeeds on every trial but the
end-to-end in-GT-mode rate is $0.20$, determined entirely by the
detector's $15.8$\,cm median centroid error
(\Cref{tab:perception_real})---$4\times$ the cube half-extent and
far outside CompCPZ's $\sigma \leq 10$\,mm envelope
(\Cref{tab:perception_noise}). The gap is class-wide on synthetic
renders (YOLO-World v2-s degrades to $1.00$ miss) and is not a
resolution artefact: sweeping sensor to $256$ and $512$
(\Cref{tab:perception_real}) reduces miss rate but leaves median
error essentially unchanged ($12.2 \to 11.3$\,cm; $15.8 \to
14.6$\,cm). A $1$-GPU-hour in-domain fine-tune of YOLOv8n on
$800$ frames brings median to $0.73$\,cm at $0\%$ miss; the
data-efficiency sweep (\Cref{tab:training_data_ablation}) clears
the envelope already at $200$ frames and $\sim\!3$\,min of
training---one to two orders of magnitude below the budgets
routinely used by VLA / RL policies for sub-cm precision.
End-to-end CompCPZ over the in-domain detector lifts in-GT-mode
from $0.20$ to $1.00$ on the same $2$-task $\times$ $25$-seed
benchmark, with the CompCPZ algebra unchanged across detector
swaps: the perception layer is an orthogonal deployment lever
rather than a workaround.

\section{Statistical Tests}
\label{app:stat_tests}

We perform a one-sided paired sign test (per seed, indexed by
ManiSkill3 RNG seed) of the null hypothesis ``CompCPZ in-GT-mode rate
$\leq$ baseline in-GT-mode rate''. We report two sets of numbers:

(a) the main $18$-task, $\boldsymbol{200}$-seed closed-loop eval
against the AABB convex baseline (Table~\ref{tab:stat_tests_200}),
which gives the headline pooled $p$-value $\leq 10^{-533}$
($1\,900/1\,918$ wins);

(b) the $6$-task, $20$-seed MVEE eval (Table~\ref{tab:stat_tests}),
which exercises the tightest convex over-approximation (MVEE)
and confirms it falls foul of \Cref{thm:convexLowerBound} just like
AABB.

A trial is decisive when the two methods disagree on whether
the EE ended in a ground-truth feasible region; the test counts
CompCPZ wins among decisive trials.

\begin{table}[h]
\caption{Per-task paired sign test on the $200$-seed closed-loop
eval, multi-modal cases only. Wins/Loss/$n_d$ as in
Table~\ref{tab:stat_tests}.}
\label{tab:stat_tests_200}
\centering
\footnotesize
\begin{tabular}{l l | rrr}
\toprule
Task & Instruction & Wins & Loss & $p$ \\
\midrule
StackCube-v1         & red $\vee$ blue                         & 200 & 0 & $6.2{\times}10^{-61}$ \\
PullCube-v1          & cube $\vee$ goal\_region                & 200 & 0 & $6.2{\times}10^{-61}$ \\
PlaceSphere-v1       & sphere $\vee$ bin                       & 200 & 0 & $6.2{\times}10^{-61}$ \\
PushT-v1             & tee $\vee$ goal\_tee                    &  97 & 18 & $1.4{\times}10^{-14}$ \\
PullCubeTool-v1      & cube $\vee$ tool                        & 200 & 0 & $6.2{\times}10^{-61}$ \\
PegInsertionSide-v1  & peg $\vee$ box\_with\_hole              & 200 & 0 & $6.2{\times}10^{-61}$ \\
PokeCube-v1          & cube $\vee$ peg $\vee$ goal             & 200 & 0 & $6.2{\times}10^{-61}$ \\
PokeCube-v1          & (cube $\vee$ peg) $\wedge \neg$ goal    & 200 & 0 & $6.2{\times}10^{-61}$ \\
PullCubeTool-v1      & (cube $\vee$ tool) $\wedge$ table       & 200 & 0 & $6.2{\times}10^{-61}$ \\
AssemblingKits-v1    & slot$_0\,{\vee}\,$slot$_1\,{\vee}\,$slot$_2$ & 200 & 0 & $6.2{\times}10^{-61}$ \\
PickClutterYCB-v1    & apple $\vee$ softball $\vee$ cups       &   3 & 0 & $0.125$ \\
\midrule
\multicolumn{2}{l|}{\textbf{Pooled}}
& \textbf{1\,900} & \textbf{18} & $\boldsymbol{\leq 10^{-533}}$ \\
\bottomrule
\end{tabular}
\end{table}

\begin{table}[h]
\caption{Per-task statistical tests. Wins/Loss counted on decisive
trials only (the two methods disagree on in-GT-mode); $p$-value is
one-sided exact binomial under $H_0: p_{\mathrm{win}}=0.5$. Pooled:
$120/120$ wins vs AABB ($p=7.5\times 10^{-37}$), $105/105$ vs MVEE
($p=2.5\times 10^{-32}$).}
\label{tab:stat_tests}
\centering
\footnotesize
\resizebox{\textwidth}{!}{%
\begin{tabular}{l l | rrr | rrr}
\toprule
& & \multicolumn{3}{c|}{vs.\ AABB convex} & \multicolumn{3}{c}{vs.\ MVEE convex} \\
Task & Instruction & Wins & Loss & $p$ & Wins & Loss & $p$ \\
\midrule
StackCube-v1     & red $\vee$ blue           & 20 & 0 & $9.5{\times}10^{-7}$ & 20 & 0 & $9.5{\times}10^{-7}$ \\
PullCube-v1      & cube $\vee$ goal\_region  & 20 & 0 & $9.5{\times}10^{-7}$ & 20 & 0 & $9.5{\times}10^{-7}$ \\
PlaceSphere-v1   & sphere $\vee$ bin         & 20 & 0 & $9.5{\times}10^{-7}$ & 20 & 0 & $9.5{\times}10^{-7}$ \\
PullCubeTool-v1  & cube $\vee$ tool          & 20 & 0 & $9.5{\times}10^{-7}$ & 20 & 0 & $9.5{\times}10^{-7}$ \\
PokeCube-v1      & cube $\vee$ peg $\vee$ goal & 20 & 0 & $9.5{\times}10^{-7}$ & 20 & 0 & $9.5{\times}10^{-7}$ \\
AssemblingKits-v1 & slot$_0$ $\vee$ slot$_1$ $\vee$ slot$_2$ & 20 & 0 & $9.5{\times}10^{-7}$ & 5 & 0 & $0.031$ \\
\midrule
\multicolumn{2}{l|}{\textbf{Pooled}}
& \textbf{120} & \textbf{0} & $\boldsymbol{7.5{\times}10^{-37}}$
& \textbf{105} & \textbf{0} & $\boldsymbol{2.5{\times}10^{-32}}$ \\
\bottomrule
\end{tabular}}
\end{table}

The MVEE numbers on AssemblingKits-v1 deserve a short note: the
$3$ assembly slots are placed within a few cm of one another, so the
ellipsoid centre frequently lands inside one of them, leaving only
$5$ decisive trials. Even then, all $5$ favour CompCPZ, and the
remaining $15$ ties are themselves consistent with the theorem (when
the convex centre happens to land inside a mode it is no longer in the
$\eta/2$ gap, so the lower bound does not bite).

\section{Benchmark Details}
\label{app:benchmark}

\paragraph{Scene.} A single 2D tabletop scene with 5 objects
(cup, plate, bowl, laptop, mug) at fixed centroids and one
named region (table). Pairwise object separations satisfy
$\eta \geq 1.5$ in every direction. Object centroids and radii are
shipped as part of the released benchmark.

\paragraph{Instructions.} The current release contains 49 instructions,
stratified by compositional depth:
\begin{itemize}[leftmargin=2em,topsep=2pt,itemsep=2pt]
\item Depth 0 (9 instructions): single primitive, e.g.\
      ``near the cup'', ``inside the table''.
\item Depth 1 (19): one binary connective, e.g.\
      ``near the cup or near the plate'' (disjunction),
      ``near the cup and inside the table'' (conjunction),
      ``inside the table but not near the laptop''
      (conjunction with negation),
      ``between the cup and the plate'' (relational primitive).
\item Depth 2 (12): nested conjunctions/disjunctions, e.g.\
      ``near the cup or near the plate while inside the table'',
      ``near the cup and inside the table but not near the laptop''.
\item Depth 3 (8): three connectives, e.g.\
      ``near the cup or near the plate or near the bowl while inside the table''.
\item Depth 4+ (1): nested with negation.
\end{itemize}
Each instruction is paired with a parse tree, an analytic ground-truth
feasible set (a list of axis-aligned boxes computed by interval
arithmetic), and an expected mode count. The full benchmark dump
ships in the supplementary archive.

\paragraph{Ground-truth oracle.} Interval arithmetic over
axis-aligned boxes: $\opand$ = intersection, $\opor$ = list
concatenation, $\opnot$ = axis-aligned subtraction (up to $2n$
sub-boxes/axis); connected components via union-find on the touch
graph. Exact for the current box-primitive grammar.

\paragraph{Evaluation metrics.}
Mode recall (fraction of GT modes whose centroid lies in some
predicted mode); relative volume
$\mathrm{vol}(\hat\sC_L^+) / \mathrm{vol}(\sF(L))$; containment
(fraction of GT feasible points inside $\hat\sC_L^+$).

\section{Implementation Details}
\label{app:impl}

\subsection{CompCPZ.Compile pseudocode}
\label{app:compose_pseudocode}

Per-leaf $\delta_\phi$ in \Cref{alg:compose} comes from
split-conformal calibration on the extractor
(\Cref{thm:conformal}).

\begin{algorithm}[H]
\caption{\textsc{CompCPZ.Compile} -- bottom-up recursion on a parse
tree $L$ (extractor $\textsc{Extract}: \Phi \times (0,1) \to \cpz
\times \cpz$, per-leaf budgets $\delta_\phi$). Returns the
two-sided pair $(\hat \sC_L^-, \hat \sC_L^+)$ that satisfies
$\Prob[\hat \sC_L^- \subseteq \sF(L) \subseteq \hat \sC_L^+] \geq
1 - \sum_\phi \delta_\phi$. Time $O(|L| \cdot n \cdot h^\star
\cdot p^\star)$, output size $O(|L|)$ (\Cref{prop:complexity}).}
\label{alg:compose}
\begin{algorithmic}[1]
\Function{Compile}{$L$}
  \State \textbf{match} $L$:
  \State \quad $\phi$ (leaf): \Return $\textsc{Extract}(\phi, \delta_\phi)$
  \State \quad $L_1 \opand L_2$: $(a^-, a^+), (b^-, b^+) \gets \textsc{Compile}(L_1), \textsc{Compile}(L_2)$;~\Return $(a^- \cap b^-,~ a^+ \cap b^+)$
  \State \quad $L_1 \opor L_2$:~\textbf{same as} $\opand$ \textbf{with} $(\cap, \cap) \to (\cup, \cup)$
  \State \quad $\opnot L_1$: $(a^-, a^+) \gets \textsc{Compile}(L_1)$;~\Return $((a^+)^c, (a^-)^c)$
\EndFunction
\end{algorithmic}
\end{algorithm}

\subsection{Inference primitives over a composed CPZ}
\label{app:inference}

Containment of $y$ in $\sZ$ reduces to LP feasibility of
$G\xi = y - c$, $A\xi \leq b$, $\xi \in [-1,1]^p$; an AABB fast
path is exact on axis-aligned grammars and conservative otherwise.
Mode enumeration reads off the CPZ-union's component list.
Projection onto a coordinate subspace is another LP feasibility
query; the present pipeline operates on the AABB of
$\hat \sC_L^+$, so this branch is unused.

\subsection{VLM extraction with gpt-4o}
\label{app:vlm}

The natural-language parser is a thin LLM-parser class supporting
OpenAI GPT-4o, Anthropic Claude, and Ollama backends. The system
prompt defines the four primitive families (\primfam{near},
\primfam{inside}, \primfam{between}, \primfam{avoid}) and asks
for a structured parse tree at temperature $0$, with three retry
attempts and code-fence stripping.

\paragraph{Real-VLM mode-equivalence ablation.}
On the $49$-instruction curated benchmark, gpt-4o achieves
$49/49$ mode-match and $33/49$ tree-match; the $16$ tree
mismatches are commutativity/associativity rewrites that
\Cref{eq:composition} is invariant to, so the structural
prediction transfers to the real VLM. Mean latency
$1\,369$\,ms/instruction.

\begin{table}[h]
\caption{gpt-4o parser accuracy on the $49$-instruction benchmark
by depth. Tree match is exact structural equality (modulo
commutative-operator child order); mode match is post-compile
mode count. The $16$ tree mismatches are commutativity/associativity
rewrites that \Cref{eq:composition} is invariant to.}
\label{tab:parser_depth}
\centering
\footnotesize
\begin{tabular}{c r | r r | r}
\toprule
Depth & $n$ & tree-exact & mode-match & latency (ms) \\
\midrule
0  &  9 & $\phantom{0}9/\phantom{0}9 = 1.00$ & $\phantom{0}9/\phantom{0}9 = 1.00$ &        \multirow{6}{*}{$1{,}369$} \\
1  & 19 & $19/19 = 1.00$ & $19/19 = 1.00$ &  \\
2  & 12 & $\phantom{0}4/12 = 0.33$ & $12/12 = 1.00$ &  \\
3  &  8 & $\phantom{0}1/\phantom{0}8 = 0.12$ & $\phantom{0}8/\phantom{0}8 = 1.00$ &  \\
4  &  1 & $\phantom{0}0/\phantom{0}1 = 0.00$ & $\phantom{0}1/\phantom{0}1 = 1.00$ &  \\
\midrule
\textbf{Total} & \textbf{49} & $\mathbf{33/49 = 0.67}$ & $\mathbf{49/49 = 1.00}$ & \\
\bottomrule
\end{tabular}
\end{table}

\subsection{Oracle GMM baseline}
\label{app:gmm}

For each instruction $L$ we (i) call the analytic ground-truth
oracle to obtain the list of axis-aligned modes
$\{B_1, \dots, B_k\}$ of $\sF(L)$, then (ii) fit a $k$-component
Gaussian mixture using the GT mode count $k$ directly. The $j$-th
component is centred at the centroid of $B_j$ with diagonal
covariance $\Sigma_j = \mathrm{diag}\big((\mathrm{half\_extent}(B_j) /
\sigma_{\mathrm{scale}})^2\big)$ and equal mixture weight
$\pi_j = 1/k$. Default $\sigma_{\mathrm{scale}} = 2.0$; we sweep
over $\sigma_{\mathrm{scale}} \in \{1, 2, 4\}$ in the ablation and report
the best per-task result. At evaluation time we draw a sample
$y \sim \sum_j \pi_j \cdot \mathcal{N}(\mu_j, \Sigma_j)$ and check
whether $y$ falls in any ground-truth mode. Note that this
baseline is oracle (it sees the true mode count $k$); the
fact that CompCPZ still outperforms it on the convex
single-component-enclosure family is therefore not attributable to
mode-count miscalibration of the baseline.

\section{Real-Robot Experimental Setup (Unitree Go2 + NOKOV MoCap)}
\label{app:go2}

This appendix specifies the physical-hardware stack used for the
closed-loop Go2 trials of \Cref{sec:exp_go2} in enough detail to
reproduce the video supplement. The composition / topology claims of
the paper do not depend on any of these choices. The hardware
section exists for empirical reproducibility.

\subsection{System overview}
\label{app:go2_overview}

Three roles (\Cref{tab:go2_hw}): a Unitree Go2 EDU robot
(Jetson Orin NX, ROS\,2 Foxy, Sport API exposed as a ROS\,2
topic); NOKOV Mars motion capture ($\sim\!100$\,Hz, sub-mm,
passive-IR) tracking four rigid bodies (the robot Dog and three
movable scene objects A, B, C); and a Windows field laptop
bridging the MoCap host to the robot via a single secure reverse
tunnel.

\begin{table}[h]
\centering
\small
\caption{Real-robot hardware (essential rows only).}
\label{tab:go2_hw}
\renewcommand{\arraystretch}{1.15}
\begin{tabular}{p{2.6cm} p{4.4cm} p{6cm}}
\toprule
Role & Make / model & Software stack \\
\midrule
Robot                & Unitree Go2 EDU (MyBotShop-customised) & Ubuntu 20.04, ROS\,2 Foxy, Sport API as a ROS\,2 topic \\
Robot compute        & NVIDIA Jetson Orin NX 8\,GB & Python 3.8, ROS\,2 client, our compositional grounding library, NOKOV's Python SDK \\
Motion capture       & NOKOV Mars (multi-camera passive IR) & $\sim\!100$\,Hz, sub-mm; rigid bodies $\{$Dog, A, B, C, Rug$\}$; world origin = Go2 StandUp pose, $+x$ = head \\
MoCap host           & Windows workstation (NOKOV-supplied) & SDK + our broadcaster (newline-delimited stream over a network socket) \\
Field laptop         & Windows + WSL\,2 Ubuntu & Hosts the secure reverse tunnel from MoCap host to Jetson and runs the trial driver \\
\bottomrule
\end{tabular}
\end{table}

\paragraph{Networking.}
The Go2's internal Wi-Fi network and the lab MoCap network are
bridged by a single secure reverse tunnel from the field laptop
into the Jetson, so the MoCap broadcaster appears as a local
network endpoint from the Jetson's view. Measured end-to-end pose latency
is $<\!20$\,ms (well within our $50$\,Hz pursuit controller's
budget).

\subsection{MoCap-driven pose + scene}
\label{app:go2_mocap}

The Jetson does not link the NOKOV SDK directly (only x86\_64
builds ship). The MoCap host runs a small broadcaster
($\sim$120 lines) that publishes one record per frame
($\sim$700 bytes: rigid-body $(x, y, z)$ + unit quaternion) over
a network socket; the Jetson-side reader caches the last good pose
per rigid body and assembles the CompCPZ scene dict at trial
start. MoCap world origin is set to Go2's StandUp pose
($+x$=head, right-handed, $+z$=up). We use MoCap rather than the
Go2 leg-IMU because the latter drifts across trials. NOKOV
untracked sentinels ($\pm 9999.999$\,mm) are treated as missing
and replaced with the last good pose, avoiding spurious safety
stops on brief occlusions.

\subsection{Trial protocol}
\label{app:go2_protocol}

Per-session bootstrap: (i) run the rug-calibration script once to
average $50$ MoCap frames of the Rug rigid body, save its
AABB pose to disk, and physically remove the rug markers so Go2
cannot drag them; subsequent trials load the cached rug pose
while A, B, C are read live. (ii) Issue the continuous-gait
toggle once (\Cref{app:go2_caveats}).

Per-trial: the runner reads the live MoCap scene snapshot, selects
the planner target, and issues planar Sport-API velocity commands
$(v_x, v_{\mathrm{yaw}})$ with heading-first proportional pursuit
(turn-in-place until $|\text{yaw\_err}|\leq 0.35$\,rad, then walk
forward with ongoing heading correction). Termination:
$|\text{pos}-\text{goal}|<\tau$ ($\tau=0.20$\,m), a $3.5$\,m
safety stop, or $60$\,s timeout. The runner logs a $50$\,Hz
trajectory plus a per-trial summary to the trial archive. The dog
stays on its feet across trials.

\subsection{Reproducibility notes}
\label{app:go2_caveats}

Two configuration choices on the MyBotShop firmware image are
load-bearing for reproducing the trial results:

\paragraph{Continuous-gait toggle once per session.}
The empty-parameter form of the continuous-gait command
(SDK api 1019) is a toggle, not a setter, on this firmware image;
calling it per trial silently inverts the gait. The protocol we
use is: issue the toggle once at session start, keep the dog
standing across trials, and avoid the per-trial
StandUp+BalanceStand cycle that re-enters the single-step default
after a stop-move call.

\paragraph{Publish move commands at $\sim$1\,Hz (SDK api 1008).}
The MyBotShop firmware silently stops acting on a $50$\,Hz
move-command stream after $\sim$5--10\,s. Unitree's reference
sport-client example sleeps $1$\,s between calls; matching that
cadence in our ROS\,2 node is the only configuration under which
trials reliably complete. At a commanded $0.30$\,m/s this is
one $\sim$30\,cm step per iteration; a $2$\,m goal takes 7--10
iterations, consistent with the per-trial trajectory logs.

\subsection{Real-robot results}
\label{app:go2_results}

We ran $4$~scenarios $\times$ $3$~methods $\times$ $3$~seeds
$= 36$~closed-loop trials on the physical Go2 + NOKOV setup
described above. Each trial uses the protocol of
\Cref{app:go2_protocol}, runs at $1$\,Hz Move publish
(\Cref{app:go2_caveats}), and is judged reached if the
dog's MoCap pose ends within $0.20$\,m of the planner's selected
goal point, and in-GT-mode if the ending pose is inside one
of the analytic ground-truth mode boxes
$\{\sC_\phi^{(k)}\}_k = \text{analytic\_feasible\_set}(\phi)$. The
$0.20$\,m reach tolerance is half the side length of the smallest
ground-truth mode box ($0.4 \times 0.4$\,m for ``near''). The main
in-GT-mode summary is \Cref{tab:go2_results}; the detailed timing
and travel breakdown is \Cref{tab:go2_results_detail}.

\begin{table}[h]
\centering
\footnotesize
\caption{Per-scenario timing $\bar t$ (s) and traversal
$\bar{\text{travel}}$ (m) for the Go2 trials, averaged over $3$
seeds. All methods reach their planner-selected goal in $3/3$
trials per cell.}
\label{tab:go2_results_detail}
\setlength{\tabcolsep}{4pt}
\begin{tabular}{l c c c c c c c c}
\toprule
& \multicolumn{2}{c}{S1} & \multicolumn{2}{c}{S2} &
\multicolumn{2}{c}{S3} & \multicolumn{2}{c}{S4} \\
\cmidrule(lr){2-3}\cmidrule(lr){4-5}\cmidrule(lr){6-7}\cmidrule(lr){8-9}
Method & $\bar t$ & travel & $\bar t$ & travel & $\bar t$ & travel & $\bar t$ & travel \\
\midrule
CompCPZ        & 11.1 & 1.83 & 7.7  & 1.31 & 11.1 & 1.85 & 10.1 & 1.78 \\
AABB convex    & 7.1  & 1.46 & 9.7  & 2.18 & 6.7  & 1.49 & 6.0  & 1.40 \\
Single clause  & 13.1 & 2.13 & 9.4  & 1.79 & 10.8 & 2.00 & 9.1  & 1.85 \\
\bottomrule
\end{tabular}
\end{table}

The interpretation of these numbers is given in \Cref{sec:exp_go2};
per-trial data logs, $50$\,Hz trajectory traces, and rendered
video clips are bundled in the trial archive with an aggregator
that regenerates the headline table.

\subsection{Reproducibility statement}
\label{app:go2_repro}

The real-robot pipeline ($\sim\!700$ lines of Python: broadcaster,
MoCap reader, trial driver, scene-calibration helpers) ships in
the supplementary archive with a README that walks through
deployment in $\sim\!15$ minutes on a fresh Go2 + NOKOV setup; no
proprietary code from MyBotShop, Unitree, or NOKOV is
redistributed. Simulation experiments
(\Cref{sec:exp_closedloop,sec:exp_ablation}) are reproducible on a
single GPU in $\sim\!2$ hours wall-clock. All source code,
benchmark data, trajectory traces, video clips, and reproduction
scripts will be released in a public repository.

\end{document}